\documentclass{article} % For LaTeX2e
\usepackage{iclr2024_conference,times}

\usepackage{amsmath,amsfonts,bm,amssymb,amsthm}

\usepackage{graphicx}
\usepackage{bbding}
\usepackage[ruled,linesnumbered]{algorithm2e}
\usepackage{multicol,multirow}
\usepackage{array}
\usepackage{enumerate}
\usepackage{booktabs}
\usepackage{float}
\usepackage{wrapfig}
\usepackage{caption}
\usepackage{enumitem}
\def\eqref#1{equation~\ref{#1}}
\def\1{\bm{1}}

\DeclareMathAlphabet{\mathsfit}{\encodingdefault}{\sfdefault}{m}{sl}
\SetMathAlphabet{\mathsfit}{bold}{\encodingdefault}{\sfdefault}{bx}{n}

\DeclareMathOperator*{\argmax}{arg\,max}

\newcommand{\abs}[1]{\left\vert#1\right\vert}
\newcommand{\set}[1]{\left\{#1\right\}}

\newcommand{\vct}{\text{vec}}

\newcommand{\mbc}{\mathbf{C}}
\newcommand{\mbg}{\mathbb{G}}

\newcommand{\mbx}{\mathbb{X}}

\newcommand{\mca}{\mathcal{A}}
\newcommand{\mcc}{\mathcal{C}}
\newcommand{\mce}{\mathcal{E}}
\newcommand{\mcg}{\mathcal{G}}

\newcommand{\mcs}{\mathcal{S}}
\newcommand{\mco}{\mathcal{O}}
\newcommand{\mcv}{\mathcal{V}}
\newcommand{\mcf}{\mathcal{F}}
\newcommand{\mcl}{\mathcal{L}}
\newcommand{\mci}{\mathcal{I}}

\newcommand{\mbfa}{\mathbf{A}}

\newcommand{\mbfc}{\mathbf{C}}
\newcommand{\mbfe}{\mathbf{E}}
\newcommand{\mbfk}{\mathbf{K}}
\newcommand{\mbfp}{\mathbf{P}}

\newcommand{\mbfx}{\mathbf{X}}

\newcommand{\mbff}{\mathbf{F}}
\newcommand{\hmbfc}{\hat{\mathbf{C}}}
\newcommand{\hc}{\hat{c}}

\newcommand{\eg}{e.g.}

\newcommand{\name}{M3C}
\newcommand{\namenn}{UM3C}

\newcommand{\KNN}{M3C-Local}
\newcommand{\threshold}{M3C-Global}
\newcommand{\Ranksum}{M3C-Fuse}

\newcommand{\ournn}{\textbf{UM3C}}

\newcommand{\yanq}[1]{{\color{black}{#1}}}

\newtheorem{theorem}{Theorem}
\newtheorem{proposition}{Proposition}[section]

\newtheorem{definition}[theorem]{Definition}

\usepackage{url}
\usepackage{makecell}

\usepackage[colorlinks,
            linkcolor=red,
            anchorcolor=blue,
            citecolor=brown,
            ]{hyperref}
\iclrfinalcopy

\title{M3C: A Framework towards Convergent, Flexible, and Unsupervised Learning of Mixture Graph Matching and Clustering}

\author{Jiaxin Lu$^\ast$\quad Zetian Jiang\thanks{Equal Contribution}\quad Tianzhe Wang\quad Junchi Yan \\
Department of Computer Science\\
Shanghai Jiao Tong University\\
 \texttt{\{lujiaxin,maple\_jzt,usedtobe,yanjunchi\}@sjtu.edu.cn} \\
}

\begin{document}

\maketitle

\begin{abstract}
Existing graph matching methods typically assume that there are similar structures between graphs and they are matchable. However, these assumptions do not align with real-world applications. This work addresses a more realistic scenario where graphs exhibit diverse modes, requiring graph grouping before or along with matching, a task termed mixture graph matching and clustering. We introduce Minorize-Maximization Matching and Clustering (M3C), a learning-free algorithm that guarantees theoretical convergence through the Minorize-Maximization framework and offers enhanced flexibility via relaxed clustering. Building on M3C, we develop UM3C, an unsupervised model that incorporates novel edge-wise affinity learning and pseudo label selection. Extensive experimental results on public benchmarks demonstrate that our method outperforms state-of-the-art graph matching and mixture graph matching and clustering approaches in both accuracy and efficiency. Source code will be made publicly available.
\end{abstract}
\section{Introduction}
\label{sec:intro}
Graph matching (GM) constitutes a pervasive problem in computer vision and pattern recognition, with applications in image registration~\citep{shen2002hammer}, recognition~\citep{duan2012discovering,demirci2006object}, stereo~\citep{goesele2007multi}, 3D shape matching~\citep{berg2005shape,petterson2009exponential}, and structure from motion~\citep{simon2007scene}. 
GM involves finding node correspondences between graphs by maximizing affinity scores, commonly formulated as the quadratic assignment problem (QAP), often referred to as Lawler's QAP~\citep{LoiolaEJOR07}:
\begin{equation}
\label{eq:gm_formulation}
\begin{aligned}
    \mbfx = \argmax_\mbfx \vct(\mbfx)^{\top} \mbfk \vct(\mbfx)  \quad
    s.t. \hspace{5pt} \mbfx &\in \{0,1\}^{n_1 \times n_2}, \mbfx \mathbf{1}_{n_2} \leq \mathbf{1}_{n_1}, \mbfx^{\top} \mathbf{1}_{n_1} \leq \mathbf{1}_{n_2}
\end{aligned}
\end{equation}
Here, $\mbfx$ is a permutation matrix encoding node-to-node correspondences, and $\bm{1}_n$ is an all-one vector.
The inequality constraints accommodate scenarios with outliers, addressing the general and potentially ambiguous nature of the problem.
Multiple graph matching (MGM)~\citep{yan2015consistency,YanPAMI16,JiangPAMI21} extends GM by enforcing cycle consistency among pairwise matching results. 
GM and MGM are both NP-hard, leading to the proposal of approximate algorithms, either learning-free~\citep{YanICMR16} or learning-based~\citep{YanIJCAI20}.

In GM, whether in two-graph matching or multi-graph matching, a common assumption prevails: the graphs must belong to the same category, and labels of both graphs and nodes are required. However, labeling can be costly, especially in domains like molecular design or drug discovery where domain-specific knowledge is needed. Real-world scenarios also involve mixtures of different graph types, e.g., in traffic tracing, frames may contain people, bicycles, and cars simultaneously. As a result, matching with mixed graph types is a practical challenge in its nascent stage. In this paper, we introduce Mixture Graph Matching and Clustering (MGMC), aiming to align graph-structured data and simultaneously partition them into clusters, as shown in Fig.~\ref{fig:m3c_um3c}.

Recent studies have explored mixture graph matching and clustering in two works: Decayed Pairwise Matching Composition (DPMC)~\citep{WangAAAI20} and the graduated assignment neural network (GANN)~\citep{wang2020graduated}. However, they suffer from certain drawbacks that warrant attention: 1) \textbf{Convergence}. DPMC exhibits convergence instability, while GANN has slow convergence. 2) \textbf{Rigid Structure}. DPMC relies heavily on its tree structure, and GANN tends to converge to a sub-optimal due to its hard clustering nature. 3) \textbf{Robustness}. GANN's robustness is compromised by noise, as shown in our experiments, where matching accuracy drops significantly with just two outliers and further deteriorates with more outliers.

To address these drawbacks, we propose our solution, \textbf{M}inorize \textbf{M}aximization \textbf{M}atching and \textbf{C}lustering (\textbf{\name}). \name\ enjoys convergence guarantees and is based on a convergent alternating optimization solver. 
We utilize the cluster indicator from hard clustering to represent the discrete structure used for optimization, providing better information for graphs of different modes while preserving the convergence guarantee of the algorithm.
We additionally introduce \ournn, which integrates the learning-free solver into an unsupervised pipeline, incorporating edge-wise affinity learning, affinity loss, and a pseudo-label selection scheme for higher robustness. A comprehensive comparison with previous works is summarized in the appendix Sec.~\ref{sec:comparison_with_previous_works}.

The main contributions of this paper are:
\begin{itemize}[leftmargin=*,itemsep=2pt,topsep=0pt,parsep=0pt]
    \item We present \name, a learning-free solver for the mixture graph matching and clustering problem, that guarantees convergence within the Minorize-Maximization framework, enhanced by a flexible optimization scheme enabled by the relaxed cluster indicator. This marks the first theoretically convergent algorithm for MGMC, to the best of our knowledge.
    
    \item We enhance \name \ by integrating it with an unsupervised pipeline \namenn. Edge-wise affinity learning, affinity loss, and pseudo label selection are introduced for learning quality and robustness.
    
    \item \name \ and \namenn \ outperform state-of-the-art learning-free and learning-based methods on mixture graph matching and clustering experiments. Interestingly, \namenn \ even outperforms supervised models such as BBGM and NGM, establishing itself as the top-performing method for MGMC on the utilized public benchmarks.
\end{itemize}

\section{Related Works}
\label{sec:related}
% We provide a concise overview of graph matching and graph clustering, as well as their intersection, which is an emerging research topic with significant practical implications.

\paragraph{Graph Matching}
Graph matching has gained attention recently, with various techniques explored, including spectrum, semi-definite programming (SDP), and dual decomposition~\citep{Gold1996AGA,WykPAMI04,Cho2010ReweightedRW,tian2012convergence,egozi2012probabilistic,swoboda2017study}. 
Multiple graph matching (MGM) introduces cycle consistency as regularization to encourage matching transitivity, whose methods fall into two categories: matrix factorization-based~\citep{KimSiggraph12,PachauriNIPS13,HuangSGP13,ZhouICCV15,chen2014near,LeonardosICRA17,HuCVPR18,SwobodaCVPR19} and supergraph-based approaches~\citep{YanPAMI16,yan2015consistency,JiangPAMI21,WangAAAI20}. The former enforces cycle consistency through matrix factorization, connecting all graphs with a universe-like graph for global consistency. The latter iteratively updates two-graph matchings by considering metrics along the supergraph path. Recent studies explore deep learning methods, using CNN and GNN for feature extraction and learning-free or neural network solvers for matching~\citep{ZanfirCVPR18,WangICCV19,Wang2019Neural,WangPAMI20,rolinek2020deep,yu2021deep,wang2020graduated,liu2020partial}, covering both supervised and unsupervised learning pipelines.

\paragraph{Graph Clustering} 
In this paper, we tackle graph clustering, which aims to group similar graphs. One approach involves embedding each graph into a Hilbert space and using clustering methods like k-means. Previous work~\citep{WangAAAI20,wang2020graduated} commonly use Spectral Cluster~\citep{ng2002spectral} based on pairwise affinity scores. Another approach~\citep{hartmanis1982computers,poljak1995solving,trevisan2012max,goemans1995improved} utilizes max cut~\citep{de2001randomized, festa2002randomized,poljak1995solving,trevisan2012max,goemans1995improved}, treating input graphs as nodes in a supergraph and assigning weights to edges based on pairwise scores. Alternative formulations for graph clustering include min cut formulation~\citep{johnson1993min}, normalized cuts~\citep{xu2009fast}, and multi-cuts~\citep{kappes2016multicuts,swoboda2017message}.

\paragraph{Mixture Graph Matching and Clustering} 
Matching with mixtures of graphs entails finding node correspondence and partitioning graphs into clusters. Previous works, DPMC~\citep{WangAAAI20} and GANN~\citep{wang2020graduated}, have explored this setting. GANN introduces GA-MGMC, a graduated assignment-based algorithm optimizing the MGMC problem in a continuous space, followed by projecting results to discrete matching. DPMC, a learning-free solver, constructs a maximum spanning tree on the supergraph and updates matching along the tree. Another work~\citep{BaiWSDM19} embeds the input graph for matching into an embedding vector for graph clustering. Joint matching and node-level clustering are explored~\citep{krahn2021convex}, solving node correspondence and segmenting input graphs into sub-graphs.
This paper focuses on the joint graph matching and clustering problem, a relatively new area in the literature. It represents a significant advancement in GM solvers for more open environments.

\begin{figure}[tb!]
    \begin{center}
        \includegraphics[width=\linewidth]{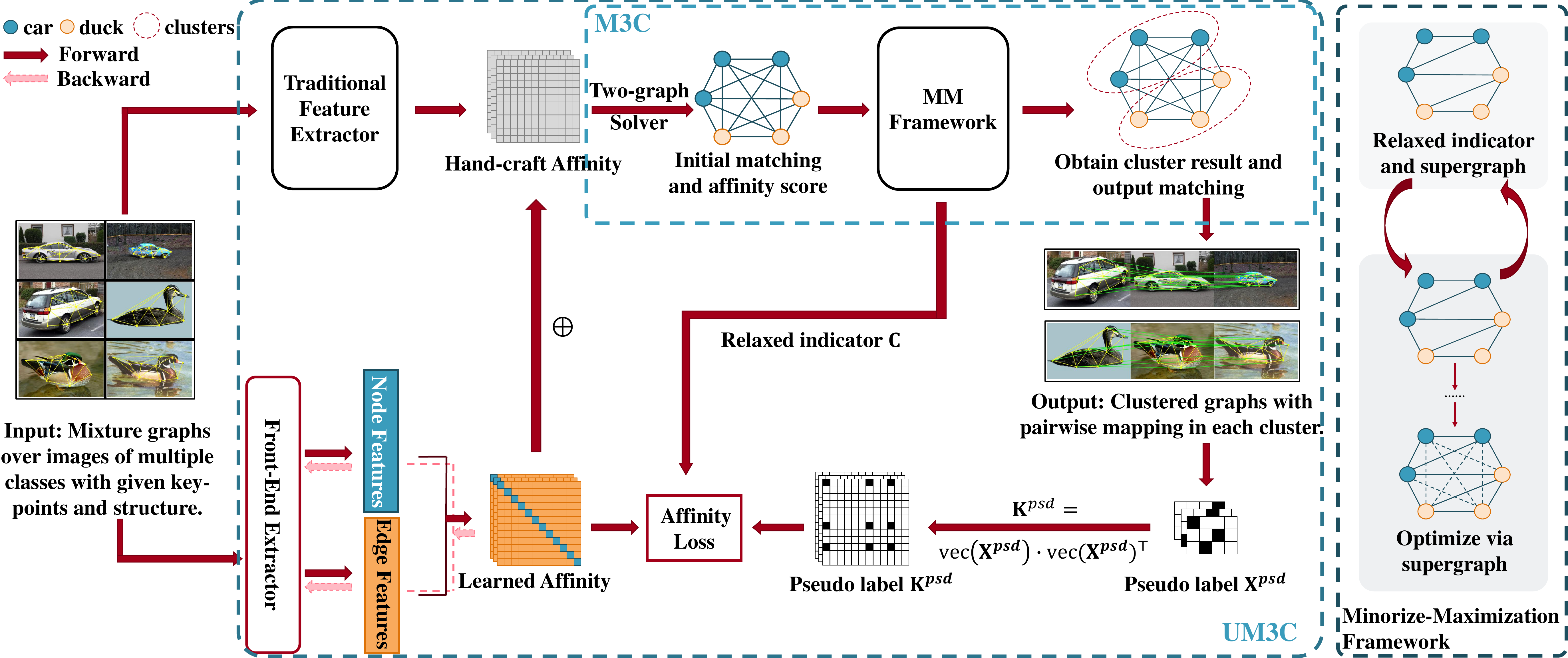}
    \end{center}
    \vspace{-5pt}
    \caption{Pipeline of \name\ and its unsupervised learning extension called \namenn. We use two clusters: ducks and cars as an illustration example. The node in the plot refers to the graph sample for matching and different colors refer to different clusters. \name \ is a leaning free solver as shown in the dotted box in the top middle with a minorize-maximization framework built on a relaxed indicator (see Section~\ref{sec:m3c}). It is extended to an unsupervised pipeline (see Section~\ref{sec:um3c}) by learning the edge-wise affinity matrix $\mbfk$ and a pseudo label selection scheme.}
    \label{fig:m3c_um3c}
    \vspace{-15pt}
\end{figure}

\section{Background and Problem Formulation} \label{sec:Preliminaries}
To facilitate the introduction of our methods, \name\ and \namenn, we introduce some preliminary concepts, definitions, and the proposed problem formulation in this section.

% \subsection{Preliminaries and Basic Concepts}
\begin{definition}[\textbf{Matching composition}]
Matching composition involves combining pairwise matching results to enhance the initial matching configuration: $\mbfx^{t+1}=\mbfx^{t}_{i k_1}\mbfx^{t}_{k_1k_2}\ldots \mbfx^{t}_{k_s j}$.
We further define the matching composition space of $\mcg_i$ and $\mcg_j$ to encompass all possible compositions between them
$\mbfp_{\mbx}(i,j) = \{ \mbfx_{i k_1} \dots \mbfx_{k_s j} | s \in \mathbb{N}^{+}, 1 \leq  k_1 \dots k_s \leq N \}$.
\end{definition}

\begin{definition}[\textbf{Supergraph}]
\label{def:supergraph}
Supergraph is a common protocol for describing multi-graph matching. The supergraph $\mcs=\set{\mcv, \mce, \mbfa}$ consists of vertices corresponding to graphs $\mcv = \set{\mcg_1, \dots , \mcg_{N}}$ and edges weighted by pairwise matching affinity scores $\mbfx_{ij}$, with adjacency $\mbfa\in \set{0,1}^{N\times N}$.

Each path on the supergraph corresponds to a matching composition. The weight of the path from $\mcg_{i}$ to $\mcg_{j}$ is defined as the affinity score of the matching composition:
  $\mbfx_{i j} = \mbfx_{i k_1} \mbfx_{k_1 k_2} \ldots \mbfx_{k_s j}$.
The matching composition space of $\mcg_i$ and $\mcg_j$ can be represented as all the paths from $\mcg_i$ to $\mcg_j$ on the supergraph:
$ \mbfp_{\mbfa}(i,j) = \{ \mbfx_{i k_1}  \dots \mbfx_{k_s j} | \forall a_{ik_1} = \dots = a_{k_sj} = 1 \}$.
\end{definition}

\begin{definition}[\textbf{Cluster indicator}]  \label{def:cluster_indicator}
The cluster indicator is defined to describe whether two graphs belong to the same cluster. It is represented by the cluster indicator matrix $\mbfc \in \set{0, 1}^{N \times N}$, where $c_{ij}=\bm{1}$ denotes ${\mcg_{i}, \mcg_{j}}$ are of the same class.
The transitive relation $c_{ij} c_{jk} \leq c_{ik}$ serves as the sufficient and necessary condition for $\mbfc$ to be a strict cluster division (see proof in the appendix). The number of clusters can be determined by the number of \textbf{strongly connected components} of $\mbfc$, denoted as $\texttt{SCC}(\mbfc)$.
\end{definition}

\paragraph{Problem Formulation}
The MGMC problem can be formulated as a joint optimization problem, where matching results maximize pair similarity to facilitate clustering, while the cluster indicator guides matching optimization in turn.
Given the set of pairwise affinity and number of clusters $N_c$, the overall objective $\mcf(\mbx, \mbfc)$ for joint matching and clustering can be written as follows:
\begin{equation}
    \begin{aligned}
        & \max_{\mathbb{X}, \mathbf{C}} \mcf(\mbx, \mbfc)=\max_{\mathbb{X}, \mathbf{C}} \frac{\sum_{ij} c_{ij} \cdot \vct(\mbfx_{ij})^{\top} \mbfk_{ij} \vct(\mbfx_{ij})}{\sum_{ij} c_{ij}} \\
        s.t. \hspace{5pt} & \mbfx_{ij} \mathbf{1}_{n_j} \leq \mathbf{1}_{n_i}, \mbfx_{ij}^{\top} \mathbf{1}_{n_i} \leq \mathbf{1}_{n_j}, \quad
        % & \mbfx_{ik} \mbfx_{kj} \leq \mbfx_{ij}, \forall c_{ik} = c_{kj} = c_{ij} = 1 \\
         c_{ik} c_{kj} \leq c_{ij}, \forall i,j,k, \hspace{2pt} \texttt{SCC}(\mbfc) = N_c.
         % \mbfx_{ij} \in \{0, 1\}^{n_i \times n_j}, c_{ij} \in \{0, 1\}.
    \end{aligned}
    \label{eq:mgmc}
\end{equation}
where $\mbx$ represents the pairwise matching matrices, and $\mbfc$ denotes the cluster indicator matrix. The first part of constraints ensures that $\mbfx_{ij}\in\set{0,1}^{n_i\times n_j}$ is a (partial) permutation matrix, and the second part requires $\mbfc\in \set{0,1}^{N\times N}$ to be a strict cluster division with $N_c$ clusters. The term $\frac{1}{\sum_{ij} c_{ij}}$ acts as a normalization factor to mitigate the influence of cluster scale. The cycle consistency within each cluster is either enforced as a constraint $ \mbfx_{ik} \mbfx_{kj} \leq \mbfx_{ij}, \ \forall c_{ik} = c_{kj} = c_{ij} = 1$, or softly encouraged by the algorithm.

\section{A Learning-free Approach: M3C}
\label{sec:m3c}

In this section, we present our learning-free algorithm, M3C. We start by converting the original problem into a Minorize-Maximization (MM) framework, a nontrivial achievement not realized before (Sec.~\ref{sec:mm_framework}). Additionally, we propose a relaxed indicator that allows for more flexible exploration by relaxing the hard constraints from independent clusters to the global and local rank of affinities (Sec.~\ref{sec:relaxed_indicator}). We finally present the full algorithm (Sec.~\ref{sec:our}).

\subsection{Converting the Problem Solving into a Minorize-Maximization Framework} \label{sec:mm_framework}
% Recall that GANN~\citep{wang2020graduated} tends to ignore the mutual optimization between clustering and matching. DPMC~\citep{WangAAAI20} device an alternative system for matching and clustering, it suffers non-convergence. To utilize the mutual optimization nature and guarantee convergence, we introduce a Minorize-Maximization (MM) framework.

Recall that DPMC~\citep{WangAAAI20} device an alternative system for matching and clustering, suffers non-convergence. To utilize the mutual optimization nature and guarantee convergence, we introduce a Minorize-Maximization (MM) framework.

We start by presenting a new objective with a single variable $\mbx$, denoted as $f(\mbx) = \mcf(\mbx, h(\mbx))$, to incorporate the MM framework into our approach. Here, $h(\mbx) = \argmax_{\mbfc}  \mcf(\mbx, \mbfc)$, represents the optimal cluster division for $\mbx$.

The MM framework works by finding a surrogate function $g(\mbx | \mbx_0) = \mcf(\mbx, h(\mbx_0))$, which minorizes the original objective function $f(\mbx)$. By optimizing this surrogate function, we can iteratively improve the objective or maintain its value. The iterative steps are as follows:
\begin{itemize}[leftmargin=*,itemsep=2pt,topsep=0pt,parsep=0pt]
\item 
Construct the surrogate function $g(\mbx | \mbx^{(t)}) = \mcf(\mbx, h(\mbx^{(t)}))$ by inferring the best cluster based on the current matching results $\mbx^{(t)}$.
\item 
Maximize $g(\mbx | \mbx^{(t)})$ instead of $f(\mbx)$, which can be solved using graph matching solvers.
\end{itemize}
The above iteration guarantees that $f(\mbx)$ is monotonic incremental and thus will converge: 
\begin{equation}\label{eq:f_monotonic}
    f(\mbx^{(t+1)}) \ge g(\mbx^{(t+1)} | \mbx^{(t)}) \ge g(\mbx^{(t)} | \mbx^{(t)}) = f(\mbx^{(t)}).
\end{equation}
Details of the proof can be found in Sec.~\ref{sec:proof_conv_mm}.

\subsection{Relaxation on Cluster Indicator}\label{sec:relaxed_indicator}
The proposed framework benefits from a convergence guarantee. However, subsequent theoretical analyses indicate that hard clustering tends to converge rapidly to a sub-optimal solution. A similar challenge is faced by GANN~\citep{wang2020graduated}, which overlooks the intrinsic differences between matched pairs and assigns them clustering weights of either $1$ or a constant $\beta$. Owing to its inherent hard clustering characteristic, GANN's performance exhibits high sensitivity to both parameter fine-tuning and the presence of outliers. This underscores our motivation to relax the hard constraints.

\begin{proposition}\label{prop:cluster}
If the size of each cluster is fixed, the hard cluster indicator converges to the local optimum in one step:
% \vspace{-3pt}
\begin{equation}
\begin{aligned}
& \mbfc^{(t)} = \mbfc^{(t+1)},  \text{if } \{N^{(t)}_{g_1}, \ldots, N^{(t)}_{g_{N_c}}\} = \{N^{(t+1)}_{g_1}, \ldots, N^{(t+1)}_{g_{N_c}}\}. 
\end{aligned}
\end{equation} 
\vspace{-15pt}
\end{proposition}
Please refer to Sec.~\ref{sec:app:hard_cluster} for the detailed proof. Two key observations are: First, the cluster indicator attains a local optimum for each cluster size group. Second, it converges within one optimization step when the target cluster size is known. Such convergence fixes the optimization space for finding optimal matching compositions, constraining the exploration of matching results.

To overcome the constrained exploration, we propose relaxing the hard constraints on the original cluster indicator. We present two relaxations focusing on the number of graph pairs for the new indicator $\hmbfc$:
\begin{equation}
\begin{aligned}
        \sum_{ij} \hc_{ij} = r \cdot N^2  \ (\text{global constraint}); \quad
        \sum_{j} \hc_{ij} = r \cdot N, \forall i  \ (\text{local constraint})
\end{aligned}
\end{equation}
Here, $r \in [0, 1]$ is a hyper-parameter that adjusts the ratio of chosen pairs. 

\begin{wrapfigure}{r}{0.45\linewidth}
\vspace{-15pt}
\begin{center}
\includegraphics[width=\linewidth]{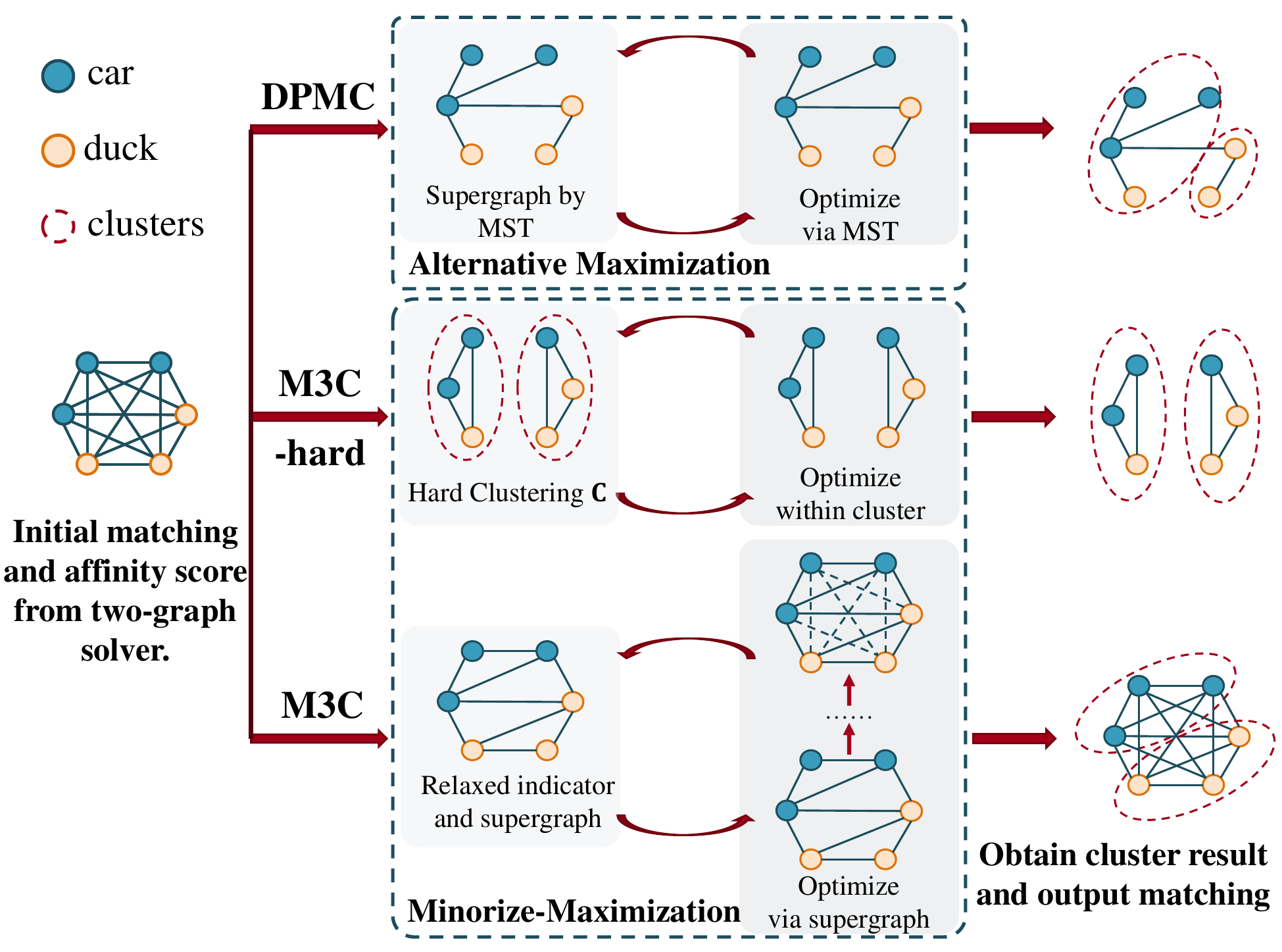}
\end{center}
\vspace{-10pt}
\caption{Three optimization structures. \textbf{DPMC}: rigid tree structure, no convergence guarantee. \textbf{M3C-Hard}: constrained exploration. \textbf{M3C}: relaxed indicator, better solutions with convergence.}
\vspace{-10pt}
\label{fig:compare}
\end{wrapfigure}
The global version limits the total number of selected graph pairs, while the local version restricts pair numbers for each graph. Both relaxations maintain the essence of the hard constraint, allowing a more flexible structure and enabling the exploration of additional matching information. We refer to this new approach as the \textbf{relaxed indicator}, which provides three advantages:
1) The relaxed indicator assesses each graph pair individually, disregarding cluster numbers and transitive relations $c_{ik} c_{kj} \leq c_{ij}$, offering flexibility in exploring matching information.
2) It enhances pseudo label selection in unsupervised learning by selecting graph pairs with higher affinity scores, as discussed in Section~\ref{sec:um3c}.
3) Compared to GANN~\citep{wang2020graduated}, which also proposes a relaxation for $\mbfc$, our approach maintains discrete constraints, resulting in faster convergence than GA-MGMC.

\subsection{The Cluster Indicator Relaxed Algorithm: \name}\label{sec:our}
We introduce the \name\ algorithm, which leverages the relaxed indicator $\hmbfc$ within the MM framework, consisting of three key parts: initialization, surrogate function construction, and maximization step. The details are outlined in Alg.~\ref{alg:main} in the appendix.

\paragraph{Initialization} 
We obtain the initial matching $\mbx^{(0)}$ using a two-graph solver, such as RRWM~\citep{Cho2010ReweightedRW}, which is both cost-effective and aligns with existing literature on multiple graph matching~\citep{YanPAMI16,JiangPAMI21}.

\paragraph{Surrogate Function Construction}
\label{sec:surrogate_construction}
With the introduction of relaxed constraints in Section~\ref{sec:relaxed_indicator}, we first present two methods for solving the optimal relaxed indicator $\hmbfc$ using the so-called global and local constraints, respectively.
\begin{equation}
    \begin{aligned}
        \hmbfc^{(t)} &= \hat{h}(\mbx^{(t-1)}) 
        = \argmax_{\hmbfc} \frac{\sum_{ij} c_{ij}}{r N^2}   \vct(\mbfx_{ij}^{(t-1)})^{\top} \mbfk_{ij} \vct(\mbfx_{ij}^{(t-1)}) \\
        s.t. &\sum_{ij} c_{ij} =  rN^2 \text{ (global constraint) or }
         \sum_{j} c_{ij} = r  N, \forall i \text{ (local constraint)}.
    \end{aligned}
    \label{eq:relax_cluster_step}
\end{equation}
In Eq.~\ref{eq:relax_cluster_step}, given $\mbx^{(t-1)}$ and $r$, the affinity scores and the normalization coefficients remain fixed. The objective is to select graph pairs with higher affinity scores. For the global constraint, we rank all graph pairs by affinity and set a threshold at the highest $r N^2$ values. For the local constraint, we rank the neighbors of each graph and select the top $rN$ pairs in the context of k-nearest-neighbor. These two algorithms are denoted as \textbf{global-rank} and \textbf{local-rank}, respectively.

Building upon the local and global schemes, we propose a fused approach named \textbf{fuse-rank}. It combines both global and local constraints by introducing a bi-level ranking. The \textbf{local} rank of a pair of graphs $(\mcg_u,\mcg_v)$ is determined as: $R_{uv} = i + j$, where $\mcg_u$ is the $i$-th nearest neighbor of $\mcg_v$, and $\mcg_v$ is the $j$-th nearest neighbor of $\mcg_u$. Subsequently, we establish a \textbf{global} threshold across all graph pairs based on $\set{R_{uv}}$. This approach allows the induced relaxed indicator $\hmbfc$ to reflect both local and global affinity relationships among graphs. We will evaluate the performance of the above three strategies in our experiments.

\paragraph{Surrogate Function Maximization}
In hard clustering, the maximization step optimizes within clusters, which is trivial to apply MGM algorithms. However, with the relaxed indicator, the optimization structure changes and requires a modification of the algorithm. We first present our approach, and then resonate why it is a well-defined generalization to address this issue.

With the supergraph mentioned in Def.~\ref{def:supergraph} and the relaxed indicator, we can construct an incomplete supergraph with edges connecting graphs of the same class. The adjacency $\mbfa$ of this supergraph corresponds numerically to the relaxed indicator $\hmbfc$, thus we simply set $\mbfa = \hmbfc$.
Let $\mbfp_{\mbfa^{(t)}}$ denote all the paths (matching composition space) based on adjacency $\mbfa$, different from traditional MGM problem, the optimization step becomes:
\begin{equation}
    \mathbb{X}^{(t)} = \argmax_{\mbfx_{ij} \in \mbfp_{\mbfa^{(t)}}(i,j)} \sum_{ij} \vct(\mbfx_{ij})^{\top} \mbfk_{ij} \vct(\mbfx_{ij}) 
    \label{eq:relax_matching_step}
\end{equation}
In Eq.\ref{eq:relax_matching_step}, the optimization space $\mbfp_{\mbfa^{(t)}}(i,j)$ increases with the number of paths between $\mcg_{i}$ and $\mcg_{j}$, making it more likely to optimize graphs of the same class. This property mirrors that of the cluster indicator and renders the maximization objective equivalent to Eq.~\ref{eq:relax_matching_step} when given the relaxed indicator $\hmbfc$.
Furthermore, when the relaxed indicator degenerates to hard cluster division, the supergraph reduces to several connected components, causing the optimization to focus solely on graphs within the same cluster. Hence, Eq.\ref{eq:relax_matching_step} is a well-defined generalization of the maximization objective.

With the proposed algorithm, the relaxed indicator $\hmbfc$ is no longer fixed, expanding the optimization space of the maximization step. This enables the maximization step to enhance the quality of the relaxed cluster. Both steps work jointly, optimizing both clustering and matching results.

In summary, we have introduced a novel learning-free approach for joint matching and clustering with a convergence guarantee. In the following section, we aim to integrate deep neural networks into our framework to enable an unsupervised learning paradigm for graph matching and clustering.

\section{Unsupervised Learning Model: \namenn} \label{sec:um3c}
We introduce \namenn, an unsupervised extension to \name, enabling joint matching and clustering within a guaranteed convergence framework. 
Recent deep learning advancements in graph matching~\citep{ZanfirCVPR18,WangICCV19,rolinek2020deep} highlight the value of learning node and edge features. However, \name's discrete optimization hinders gradient computation and machine learning integration. Previous attempts, like BBGM~\citep{rolinek2020deep}, struggled with gradient approximation and affinity generalization. GANN~\citep{Wang2019Neural}, although unsupervised, focused solely on Koopmans-Beckmann's QAP, ignoring edge features and outlier handling. \namenn\ addresses these challenges with two techniques: 1) edge-wise affinity learning and affinity loss, guided by pseudo matching label from \name\ solver; 2) Pseudo-label selection using the introduced relaxed indicator to improve pseudo label quality. We detail these techniques in the following subsections and present \name's integration into the unsupervised learning pipeline.

\subsection{Edge-wise Affinity Learning}

Our solver \name\ adopts Lawler's QAP~\citep{LoiolaEJOR07}, which embraces second-order information for enhanced robustness and performance. This requires meticulous design of the affinity matrix $\mbfk$.

Recent research on graph matching~\citep{rolinek2020deep, Wang2019Neural} adopts deep learning pipelines to compute node features $\mbff^{n}_{i}$ and edge features $\mbff^{e}_{ij}$ using VGG-16~\citep{simonyanICLR14vgg} and Spline CNN~\citep{FeyCVPR18}. A learned affinity $\mbfk^{learn}$ can be constructed from these first and second-order features as:
$\mbfk^{learn}_{u} = (\mbff^{n})^{\top} \mathbf{\Lambda} \mbff^{n}, \mbfk^{learn}_{q} = (\mbff^{e})^{\top} \mathbf{\Lambda} \mbff^{e}$. 
Here, $\mbfk_{u}$ and $\mbfk_{q}$ denote unary and quadratic affinities respectively, with $\mathbf{\Lambda}$ set to $\textbf{I}$ for stable training.

Our goal is to improve $\mbfk^{learn}$ using ground truth $\mbfx^{gt}$ (or pseudo-label $\mbfx^{psd}$). Commonly, a differentiable pipeline calculates prediction matching $\mbfx$ for applying the loss function. However, this approach conflates affinity construction and the solver, leading to customized affinities limited to the training solver and hampered generalization.

We devise a cross-entropy loss to quantify the discrepancy between two input affinity matrices:
\begin{equation}
    \mathcal{L}(\mbfk^{learn}, \mbfk^{gt})
    = \sum_{pq} K^{gt}_{pq} \log (K_{pq}^{learn}) + (1-K^{gt}_{pq}) \log (1 - K_{pq}^{learn}), \ 
    \mbfk^{gt} = \vct(\mbfx^{gt}) \cdot \vct(\mbfx^{gt})^{\top}
\end{equation}
Note that $K_{pq}$ represents the element in the $p$-th row and $q$-th column of $\mbfk$ between two graphs. This affinity loss decouples from solver effects, focusing solely on affinity quality, enhancing robustness and applicability to various solvers.

With its higher-order information, \namenn \ demonstrates significantly greater robustness against noise such as outliers, compared to GANN~\citep{wang2020graduated}, which centers only on node similarity and structural alignment.

\subsection{Unsupervised Learning using Pseudo Labels}

As previously mentioned, creating labels like $\mbfx^{gt}$ demands significant time and effort. Hence, we generate pseudo labels $\mbfx^{psd}$ using our learning-free solver \name, aiming to replace $\mbfx^{gt}$. To enhance the quality of $\mbfx^{psd}$, we propose two techniques: affinity fusion and pseudo label selection.

The quality of $\mbfx^{psd}$ hinges on input affinity $\mbfk$. However, at the beginning of the training, the learned affinity $\mbfk^{learn}$ is random and unreliable. Conversely, the hand-crafted affinity $\mbfk^{raw}$ captures only geometric information of node pairs, i.e. distances and angles, limiting its expressiveness. Consequently, we fuse both affinities, leveraging their strengths.

In particular, for its simplicity and experimental effectiveness, we linearly merge learned and hand-crafted affinities, balanced by hyperparameter $\alpha$ as $\mbfk = \mbfk^{learn} + \alpha \mbfk^{raw}$.
The hand-crafted $\mbfk^{raw}$ adheres to a standard procedure~\citep{YanPAMI16, JiangPAMI21}. Both affinities are normalized to the same scale.
This design capitalizes on the reliability of hand-crafted affinity and the expressiveness of learned affinity. In initial epochs, $\mbfk^{raw}$ enhances the quality of pseudo matching $\mbfx^{psd}$. Later, learned affinity $\mbfk^{learn}$ surpasses $\mbfk^{raw}$, further refining the final result.

Furthermore, we optimize pseudo label pair selection for loss computation, guided by the relaxed indicator $\hmbfc$.
Unlike GANN, which accumulates losses for all pairs within inferred clusters, where a single incorrect assignment affects multiple pseudo label pairs of differing categories. \namenn \ adheres to $\hmbfc$ and selects graph pairs with higher affinity rank as pseudo labels. This strategy, assuming higher affinity indicates greater accuracy, enhances pseudo affinity $\mbfk^{psd}$ quality. The overall loss is:
\begin{equation}
    \mcl_{all} = \sum_{ij} \hc_{ij} \cdot \mcl(\mbfk_{ij}^{learn}, \mbfk_{ij}^{psd}).
\end{equation}
This approach chooses more accurate matching pairs, bringing pseudo affinity closer to ground truth. Empirical validation of this approach can be found in Section~\ref{sec:ablation}.
\section{Experiments}\label{sec:exp}

\subsection{Protocols}

Experiments on all learning-free solvers were conducted on a laptop with a 2.30GHz 4-core CPU and 16GB RAM using Matlab R2020a. All learning-based experiments were carried out on a Linux workstation with Xeon-3175X@3.10GHz CPU, one RX8000, and 128GB RAM.

\paragraph{Datasets} We evaluate using two widely recognized datasets, Willow ObjectClass~\citep{ChoICCV13} and Pascal VOC~\citep{PascalVOC}. Detailed introduction and implementation of the datasets will be introduced in Sec.~\ref{sec:datasets_detail}. For convenience of notation, $N_c$ and $N_g$ (denoted as $N_c \times N_g$) represent the number of categories and graphs we selected and mixed for tests on MGMC.

\paragraph{Methods}
We present three method versions: \name-hard, \name, and \namenn. \name-hard serves as a baseline following the hard clustering MM framework (Section~\ref{sec:mm_framework}), employing Spectral Clustering and MGM-Floyd. \name\ represents the relaxed algorithm from Section~\ref{sec:our}, and \namenn\ is the unsupervised learning model described in Section~\ref{sec:um3c}, both using the \textbf{fuse-rank} scheme, if not otherwise specified.
We evaluate our methods in both learning-free and learning-based contexts. In learning-free experiments, we mainly compare \name\ with DPMC~\citep{WangAAAI20} and MGM-Floyd~\citep{JiangPAMI21}, following protocols from~\citet{WangAAAI20,wang2020graduated}. In learning-based experiments, we compare \namenn\ with unsupervised method GANN~\citep{wang2020graduated}, and supervised learning method BBGM~\citep{rolinek2020deep} and NGMv2~\citep{WangPAMI20}.

\paragraph{Evaluation Metrics}
We employ matching accuracy (MA), clustering purity (CP), rand index (RI), clustering accuracy (CA), and time cost as evaluation metrics, following prior research~\citep{WangAAAI20, wang2020graduated}. MA assesses matching performance, while CP, RI, CA represent the quality of cluster division. Detailed mathematical definitions are provided in Sec.~\ref{sec:metric_detail}. Mean results from 50 tests are reported unless specified otherwise.

\subsection{Performance on MGMC}

\begin{table*}[tb!]
	\centering
	\caption{Evaluation of matching and clustering metric with inference time for the mixture graph matching and clustering on Willow Object Class. Following the previous work~\citep{wang2020graduated}, We select Car, Duck, and Motorbike as the cluster classes.}
    \label{tab:main_test}
   \vspace{-5pt}
    \resizebox{\textwidth}{!}
    {
	\centering
\begin{tabular}{l|c|ccccc|ccccc|ccccc}
\toprule[2pt]
                 &            & \multicolumn{5}{c|}{$N_c = 3,  N_g = 8$, 0 outlier}                                             & \multicolumn{5}{c|}{$N_c = 3,  N_g = 8$, 2 outliers}                                            & \multicolumn{5}{c}{$N_c = 3,  N_g = 8$, 4 outliers}                                             \\
Model                         & Learning   & MA $\uparrow$  & CA $\uparrow$  & CP $\uparrow$  & RI $\uparrow$  & time(s) $\downarrow$ & MA $\uparrow$  & CA $\uparrow$  & CP $\uparrow$  & RI $\uparrow$  & time(s) $\downarrow$ & MA$\uparrow$   & CA $\uparrow$  & CP $\uparrow$  & RI $\uparrow$  & time(s) $\downarrow$ \\ \midrule
RRWM & free       & 0.748          & 0.815          & 0.879          & 0.871          & \textbf{0.4}         & 0.595          & 0.541          & 0.643          & 0.680          & \textbf{0.4}         & 0.572          & 0.547          & 0.661          & 0.685          & \textbf{0.6}         \\
MatchLift  & free       & 0.764          & 0.769          & 0.843          & 0.839          & 7.8                  & 0.530          & 0.612          & 0.726          & 0.730          & 10.6                 & 0.512          & 0.582          & 0.701          & 0.709          & 11.5                 \\
MatchALS      & free       & 0.635          & 0.571          & 0.689          & 0.702          & 1.3                  & 0.245          & 0.39           & 0.487          & 0.576          & 2.5                  & 0.137          & 0.383          & 0.480          & 0.571          & 2.6                  \\
CAO-C          & free       & 0.875          & 0.860          & 0.908          & 0.903          & 3.3                  & \textbf{0.727} & 0.574          & 0.678          & 0.704          & 3.7                  & \textbf{0.661} & 0.562          & 0.674          & 0.695          & 4.9                  \\
MGM-Floyd    & free       & 0.879          & \textbf{0.931} & \textbf{0.958} & \textbf{0.952} & 2.0                  & 0.716          & 0.564          & 0.667          & 0.696          & 2.3                  & 0.653          & 0.580          & 0.690          & 0.708          & 2.9                  \\
DPMC          & free       & 0.872          & 0.890          & 0.931          & 0.923          & 1.2                  & 0.672          & 0.617          & 0.724          & 0.733          & 1.4                  & 0.630          & 0.600          & 0.707          & 0.722          & 2.3                  \\
\name-hard                   & free       & 0.838          & 0.855          & 0.907          & 0.899          & 0.4                  & 0.620          & 0.576          & 0.684          & 0.705          & 0.6                  & 0.596          & 0.587          & 0.694          & 0.713          & 0.7                  \\
\textbf{\name~(ours)}           & free       & \textbf{0.884} & 0.911          & 0.941          & 0.938          & 0.5                  & 0.687          & \textbf{0.653} & \textbf{0.750} & \textbf{0.758} & 0.6                  & 0.635          & \textbf{0.646} & \textbf{0.748} & \textbf{0.753} & 1.0                  \\ \midrule[2pt]
NGMv2         & sup. & 0.885          & 0.801          & 0.843          & 0.825          & 9.0                  & 0.780          & 0.927          & 0.952          & 0.941          & 4.7                  & 0.744          & 0.886          & 0.916          & 0.906          & 4.7                  \\
BBGM     & sup. & 0.939          & 0.704          & 0.751          & 0.758          & \textbf{1.6}         & 0.806          & 0.964          & 0.977          & 0.971          & 4.8                  & 0.747          & 0.881          & 0.918          & 0.908          & 6.6                  \\ \midrule
GANN   & unsup.     & 0.896          & 0.963          & 0.976          & 0.970          & 5.2                  & 0.610          & 0.889          & 0.918          & 0.913          & 20.6                 & 0.461          & 0.847          & 0.893          & 0.881          & 30.2                 \\
\textbf{\namenn~(ours)}         & unsup.     & \textbf{0.955} & \textbf{0.983} & \textbf{0.988} & \textbf{0.988} & 3.2                  & \textbf{0.858} & \textbf{0.984} & \textbf{0.989} & \textbf{0.986} & \textbf{3.3}         & \textbf{0.815} & \textbf{0.981} & \textbf{0.987} & \textbf{0.986} & \textbf{3.6}         \\ \bottomrule[2pt]
\end{tabular}
	}
\end{table*}

\begin{table*}[tb!]
	\centering
	\caption{Evaluation of matching and clustering metric with inference time for the mixture graph matching and clustering on Pascal VOC. Mixture classes are randomly picked and average values of metrics are reported over all the combinations.}
    \label{tab:voc_main_test}
    \vspace{-5pt}
    \resizebox{\textwidth}{!}
    {
	\centering
\begin{tabular}{c|c|ccccc|ccccc}
\toprule[2pt]
                           &          & \multicolumn{5}{c|}{$N_c = 3,  N_g = 8$}                                                             & \multicolumn{5}{c}{$N_c = 5,  N_g = 10$}                                                             \\
Model                      & Learning & MA $\uparrow$             & CA $\uparrow$            & CP $\uparrow$            & RI $\uparrow$            & time(s) $\downarrow$             & MA $\uparrow$            & CA $\uparrow$            & CP $\uparrow$            & RI $\uparrow$            & time(s) $\downarrow$              \\ \midrule[2pt]
GANN                       & unsup.   & 0.2774          & 0.6949          & 0.7613          & 0.7680           & 33.785 & 0.2372          & 0.5103          & 0.5990           & 0.7816          & 64.015 \\
\textbf{\namenn(ours)}     & unsup.   & \textbf{0.4979} & \textbf{0.7015} & \textbf{0.769}  & \textbf{0.7750} & \textbf{5.2991}           & \textbf{0.4817} & \textbf{0.5551} & \textbf{0.6310}  & \textbf{0.7921} & \textbf{24.661}         \\ \midrule[2pt]
NGMv2                      & sup.     & \textbf{0.8114} & 0.7550          & 0.8083          & 0.8165          & 4.258           & 0.8210          & 0.6087          & 0.6890           & 0.8167          & 18.808          \\
BBGM                       & sup.     & 0.7919          & 0.7973          & 0.8406          & 0.8371          & \textbf{2.261}  & 0.7926          & 0.7261          & 0.7830           & 0.8656          & \textbf{8.514}   \\ \midrule
\multirow{2}{*}{\makecell[c]{BBGM(pretrained)\\+\textbf{\namenn(ours)}}} & unsup.   & 0.7928          & 0.8761          & 0.9065          & 0.9061          & 5.350             & 0.7862          & \textbf{0.7861} & 0.8320           & \textbf{0.8989} & 24.857           \\
                           & sup.     & 0.8037          & \textbf{0.8945} & \textbf{0.9212} & \textbf{0.9180}  & 5.164           & \textbf{0.7937} & 0.7850          & \textbf{0.8321} & 0.8986          & 25.752           \\ \bottomrule[2pt]
\end{tabular}
}
\end{table*}
We conduct mixture graph matching and clustering experiments on Pascal VOC and Willow ObjectClass, as detailed in Table~\ref{tab:main_test} and Table~\ref{tab:voc_main_test}. For Pascal VOC, we explore two size settings: 3 clusters $\times$ 8 images and 5 clusters $\times$ 10 images, where clusters and images are randomly chosen from the dataset. In the Willow ObjectClass dataset, we use the 3 clusters $\times$ 8 images setting while investigating the impact of outliers. This setting follows the previous work~\citet{wang2020graduated}, selecting Car, Duck, and Motorbike as the cluster classes, with images randomly sampled for each class.

Table~\ref{tab:main_test} presents the performance of our learning-free solver, \name, demonstrating its competitiveness compared to other learning-free algorithms. \name achieves top matching accuracy (over $1\%$ gain) in the settings without outliers. As the number of outliers increases, \name's strength in clustering metrics becomes apparent, with gains of $3\%$ - $5\%$ in clustering accuracy. \name\ also significantly outperforms \name-hard, affirming the effectiveness of our designed relaxed indicator.

Our unsupervised model, \namenn, excels in both matching and clustering tasks. When compared to the peer method GANN, \namenn\ showcases superior performance on both Pascal VOC and Willow ObjectClass, achieving remarkable improvements of $5.9\%$ - $24.45\%$ in matching accuracy and $0.66\%$ - $2\%$ in clustering metrics. This advantage becomes more pronounced in the presence of outliers, underscoring the robustness of our unsupervised approach. Moreover, \namenn\ even outperforms supervised models on Willow ObjectClass, as evident in the bottom section of Table~\ref{tab:main_test}. When combining BBGM(pretrain) + \namenn(finetune) on Pascal VOC, we observe a significant performance boost of $3.33\%$ - $7.88\%$ in clustering metrics over the backbone. Additionally, \namenn\ maintains impressive time-efficiency, consuming only $1.3$ - $3\times$ as much time as the two-graph matching NGMv2 and BBGM, while being $2.5$ - $6.3 \times$ faster than the peer method GANN.

% \todo{time?} However, it's worth noting that \namenn\ faces time efficiency challenges due to its for-loop algorithm, resulting in suboptimal time consumption that escalates quadratically.
\begin{figure}[tb!]
    \centering
    \includegraphics[width=\textwidth]{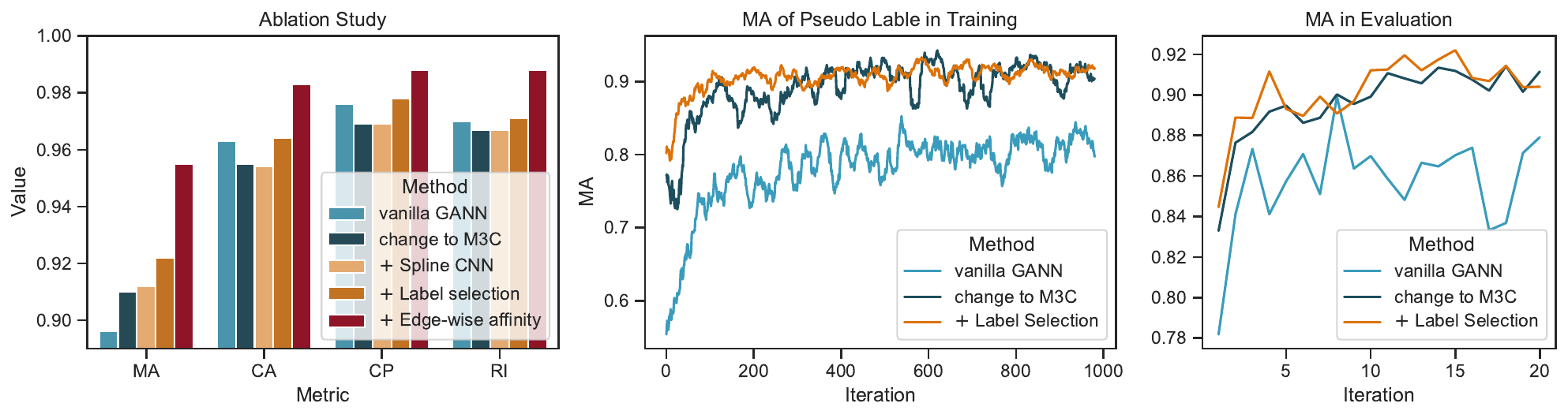}
    \vspace{-15pt}
    \caption{Ablation study of \name~by adding components under $3\times 8$ Willow ObjectClass dataset without outliers. Left: Evaluation on matching and clustering performance. Right: Quality of pseudo labels during training and evaluation by iteration.}
    \label{fig:ablation}
    \vspace{-15pt}
\end{figure}

\subsection{Ablation Study}\label{sec:ablation}

We evaluate \namenn\ on Willow ObjectClass with 3 clusters $\times$ 8 images, excluding outliers, to assess the effectiveness of different model components. We establish a baseline by substituting \name\ for the GA-GM solver in GANN. Given that GANN lacks Spline CNN for feature refinement, we also investigate the impact of introducing Spline CNN in the unsupervised method. The effectiveness of label selection and edge-wise affinity learning is validated by adding each component successively. 

Pseudo label selection results, represented by the matching accuracy of pseudo labels during both training and evaluation, are shown on the right side of Fig.~\ref{fig:ablation}. It demonstrates a $5\%$ improvement over baseline \name\ during the early training stage (first 100 iterations), affirming its ability to select pseudo labels closer to the ground truth.

Regarding edge-wise affinity learning, the left side of Fig.~\ref{fig:ablation} illustrates its significant contribution to our model, highlighting the inadequacy of hand-crafted affinity and the necessity of learning edge-wise affinity. Spline CNN further enhances matching accuracy. Notably, our baseline method already outperforms GANN, attributed to our solver \name.

\subsection{Additional Experiments}

We provide an index of additional experiments in the appendix for your reference. (\ref{sec:varying_setting}) Varying cluster number and cluster size: results on more diverse settings on Willow ObjectClass. (\ref{sec:rank_schemes}) Comparison of different ranking schemes. (\ref{sec:comparison_of_clustering_alg}) Comparison of different clustering algorithms. (\ref{sec:convergence_study}) Convergence study of M3C. (\ref{sec:hyperparameter}) Hyperparameter study of M3C. (\ref{sec:generalization_of_k}) Generalization test of learned affinity $\mbfk^{learn}$. (\ref{sec:more_visualization}) Visualization of matching results.

\section{Conclusion and Outlook}
\label{sec:con}
We have presented a principled approach for jointly solving graph matching and clustering in scenarios involving mixed modes - a challenge that extends beyond traditional graph matching. Our learning-free solver, \name, aligns with the minorize-maximization paradigm and introduces a relaxed cluster indicator to improve algorithm flexibility. Additionally, we integrate the M3C solver into an unsupervised learning pipeline, resulting in \namenn, with edge-wise affinity learning and pseudo label selection schemes. Remarkably, our methods outperform all state-of-the-art methods in the context of joint graph matching and clustering, which we believe is a practical setting to advance the research of graph matching. There are some future directions. First, while our approach employs a relaxed cluster indicator, it is constructed through traditional algorithms. There is an opportunity to harness graph learning techniques to enhance the structure's quality. Second, unsupervised learning still grapples with limitations of hand-craft affinities.

\subsubsection*{Acknowledgments}
This work was partly supported by National Key Research and Development Program of China (2020AAA0107600), National Natural Science Foundation of China (61972250, 72061127003), and Shanghai Municipal Science and Technology (Major) Project (2021SHZDZX0102, 22511105100).

\bibliographystyle{iclr2024_conference}

% \bibliography{main.bbl}

\newpage

\appendix
\section*{Appendix}

\section{Notation}

We first present all notations used in this paper for a better understanding of proposed algorithms and to facilitate the following discussion.

\begin{table*}[h] 
    \centering
    \caption{Main notations and description used in this paper.}
    \label{tbl-notation}
    \resizebox{\linewidth}{!}{
    \begin{tabular}{c|m{360pt}}
        \toprule
        \textbf{Notations} & \textbf{Descriptions} \\
        \midrule
        $N$    & Number of input graphs. \\
        \midrule
        $\mbg$ & $\mbg$ is a set of $N$ graphs $\mbg = \{\mcg_{1} \dots \mcg_{N}\}$. \\
        \midrule
        $\mcg$ & Graph to be matched with vertex set $\mcv_{\mcg}$ and edge set $\mce_{\mcg}$. \\
        % \hline
        % $v$ & $v$ denotes the node on the graph. \\
        \midrule
        $\mbx$ & $\mbx$ denotes all the possible pairwise matching results in graph set $\mbg$: $\mbx = \{\mbfx_{ij}\}_{1\leq i,j \leq N}$. \\
        \midrule
        $\mbfx_{ij}(\mbfx_{\mcg_{i},\mcg_{j}})$ & $\mbfx_{ij} (\mbfx_{\mcg_{i},\mcg_{j}}) \in \{0, 1\}^{n_i \times n_j}$ denotes the pairwise matching results between $\mcg_{i}$ and $\mcg_{j}$. \\
        \midrule
        $\mbfk_{ij}$ & $\mbfk_{ij} \in \mathbb{R}^{n_i n_j \times n_i n_j}$ denotes the affinity matrix between $\mcg_{i}$ and $\mcg_j$. Its diagonal and off-diagonal elements store the node-to-node and edge-to-edge affinities, respectively. \\
        \midrule
        $J_{ij}$ & $J_{ij} = \vct(\mbfx_{ij})^{\top} \mbfk_{ij} \vct(\mbfx_{ij})$ denotes the affinity score of graph pair $(\mcg_i, \mcg_j)$. \\
        \midrule
        $\mbfc$ & $\mbfc \in \{0, 1\}^{N \times N}$ is the cluster indicator of $N$ graphs. $c_{ij} = 1$ if $\mcg_i$ and $\mcg_j$ belongs to the same cluster, and $c_{ij} = 0$ otherwise. \\
        \midrule
        $n$, $n_o$& number of nodes in a graph and the number of its outliers\\
        \midrule
        $N_{c}, N_{g}$ & $N_{c}$ denotes the number of clusters. The size of each cluster is represented as $\{N_{g_1}, N_{g_2}, \ldots, N_{g_{N_c}}\}$. When the size of clusters is the same, we shorthand it as $N_{g}$. \\
        \midrule
        $\hmbfc$ & $\hmbfc \in \{0, 1\}^{N \times N}$ is the relaxed cluster indicator by solving relaxed constraints.\\
        \midrule
        $\mbfa$ & $\mbfa \in \{0, 1\}^{N \times N}$ denotes the adjacency matrix for supergraph (Def.~\ref{def:supergraph}).  \\
        \midrule
        $\vct(\cdot)$ & Column-vectorized operation of given matrix. \\
        \midrule
        $\texttt{SCC}(\cdot)$ & The number of strong connection components of the input cluster. \\
        \midrule
        $r$ & Hyperparameter for cluster relaxation. $r$ stands for the ratio of graph pairs to choose. \\
        \midrule
        $\alpha$ & Hyperparameter for affinity construction. $\alpha$ is the weight of the hand-craft affinity. \\
        \bottomrule
    \end{tabular}
    }
\end{table*}

\section{Comparison with Previous Works}
\label{sec:comparison_with_previous_works}
We underscore some key differences between our proposed method and two previous works, DPMC~\citep{WangAAAI20} and GANN~\citep{wang2020graduated, wang23pamiunsupervised}, focusing on MGMC in Table~\ref{tab:comp_gann_m3c}. 
This comparison is to provide a comprehensive illustration of our novel contributions and substantial advancements in this field.

Some key points of differentiation are worth emphasizing:
\begin{itemize}[leftmargin=*,itemsep=2pt,topsep=0pt,parsep=0pt]
\item Both approaches employed rigid optimization structures (a tree and a fully connected graph), while \name\ utilizes a discrete and partly connected supergraph, enhancing efficiency and flexibility in matching and clustering.
\item DPMC lacks a convergence guarantee, and GANN requires hundreds of iterations for convergence, whereas \name\ achieves convergence with fewer iterations.
\item While GANN claimed to jointly optimize matching and clustering, empirical results reveal its inability to adapt clustering once it incorporates clustering during matching, which becomes a variant of hard clustering. In contrast, \name\ demonstrates the crucial capability to jointly adjust its relaxed indicator alongside matching, a defining feature of a true joint optimization framework.
\item GANN is rooted in Koopmans-Beckmann’s QAP~\citep{KBQAP57}, which only considers structural similarity during matching, limiting its generalizability to edge feature learning. In contrast, \namenn\ utilizes a more generalized Lawler's QAP~\citep{LoiolaEJOR07} variation and introduces edge-wise affinity learning, making it more versatile.
\end{itemize}

\begin{table*}[h]
\centering
\caption{Comparison of the existing three works designated for mixture graph matching and clustering: DPMC~\citep{WangAAAI20}, GANN~\citep{wang2020graduated}, and ours: M3C/UM3C. 
% In GANN node embedding is performed for affinity learning. \yanr{We design the discrete and partly connected supergraph for efficient and flexible graph matching and clustering. While the tree structure required in DPMC can be too strict and the continuous-relaxed edge weights adopted in GANN can be less efficient and requires more iteration and additional annealing hyperparameter to converge.}
}
\label{tab:comp_gann_m3c}
\resizebox{\linewidth}{!}{
\begin{tabular}{r|ccc}
\toprule
Methods & DPMC & GANN & M3C/UM3C (ours) \\ \midrule
Optimization Space & Discrete & Continuous & Discrete \\
Framework & - & Graduated Assign & Minorize-Maximization \\
Supergraph Structure & discrete tree & real-value edge fully connected & discrete and partly-connected \\
Joint Optimize & $\times$ & $\checkmark$ & $\checkmark$ \\
Convergence & $\times$ & slow & fast \\
Affinity Learning & - & node & node and edge \\
Pseudo Label & - & unselected & selected w/ relaxed indicator \\
Empirical Robustness & mediocre & hardly work given 2+ outliers & 30\% accuracy improvement over GANN \\ \bottomrule
\end{tabular}}
\end{table*}

\section{Convergence Analysis of M3C}
\subsection{Cluster Division and Cluster Indicator}
\label{sec:app:cdci}
In this section, we show the relationship between cluster division and cluster indication. We prove that the \textbf{transitive relation} $c_{ij} c_{jk} \leq c_{ik}$ is the sufficient and necessary condition for $\mbfc$ to be a strict cluster division, which is proposed in Def.~\ref{def:cluster_indicator} of Sec.~\ref{sec:Preliminaries}. 

\begin{proposition} Transitive relation $c_{ij} c_{jk} \leq c_{ik}$ is the sufficient and necessary condition for $\mbfc$ to be a strict cluster division, where $\mbfc \in \{0, 1\}^{N \times N}$ and $c_{ij}$ denotes whether $\mcg_i$ and $\mcg_j$ are belong to the same category.
\end{proposition}

\begin{proof}
\textbf{Necessary condition}: $\mbfc$ is a strict cluster division $\Longrightarrow c_{ij} c_{jk} \leq c_{ik}$.
\begin{itemize}[leftmargin=*,itemsep=2pt,topsep=0pt,parsep=0pt]
    \item When $c_{ij}=1$ and $c_{jk}=1$, $\mcg_i$ and $\mcg_j$ are of the same class, where $\mcg_j$ and $\mcg_k$ are of the same class, too. Therefore, $\mcg_i$ and $\mcg_k$ are of the same class, which means $c_{ik}=1$ and $c_{ij} c_{jk} \leq c_{ik}$ holds.
    \item When one of $c_{ij}$ and $c_{jk}$ is equal to 1 and another equals 0, we assume  $c_{ij}=1$ and $c_{jk}=0$ without loss of generality. That is, $\mcg_i$ and $\mcg_j$ are of the same class while $\mcg_j$ and $\mcg_k$ are not. Therefore, $\mcg_i$ and $\mcg_k$ are not of the same class either, which means $c_{ik}=0$ and $c_{ij} c_{jk} = 0 \leq c_{ik}$ also holds.
    \item When $c_{ij}=0$ and $c_{jk}=0$, $\mcg_i$ and $\mcg_j$ are not of the same class, while $\mcg_j$ and $\mcg_k$ are neither of the same class. In that case, we cannot tell the relationship between $\mcg_i$ and $\mcg_k$, so $c_{ik} = 0 / 1$. It still holds that $c_{ij} c_{jk} \leq c_{ik}$.
\end{itemize}

\textbf{Sufficient condition}: $c_{ij} c_{jk} \leq c_{ik} \Longrightarrow$  $\mbfc$ is a strict cluster division.
\begin{itemize}[leftmargin=*,itemsep=2pt,topsep=0pt,parsep=0pt]
    \item If $\mcg_i$ and $\mcg_j$, $\mcg_j$ and $\mcg_k$ are both of the same class, $c_{ij} c_{jk} = 1 \leq c_{ik}$, which means $c_{ik} = 1$ and $\mcg_i$ and $\mcg_k$ are of the same class.
    
    \item If $\mcg_i$ and $\mcg_j$ are of the same class, and $\mcg_j$ and $\mcg_k$ are of different classes, we have $c_{ij} = c_{ji} = 1$ and $c_{jk} = 0$. We find that $c_{ik} = c_{ji} c_{ik} \leq c_{jk} = 0$, thus, $c_{ik} = 0$ and $\mcg_i$ and $\mcg_k$ are not of the same class.
\end{itemize}
Above all, transitive relation $c_{ij} c_{jk} \leq c_{ik}$ is the sufficient and necessary condition for $\mbfc$ to be a strict cluster division.
\end{proof}

\subsection{Proof of the Convergence of Minorize-Maximization Framework}
\label{sec:proof_conv_mm}
In this section, we prove the convergence of the MM Framework. 

The objective function $\mcf(\mbx, \mbfc)$ is defined in Eq.~\ref{eq:mgmc} as:
\begin{equation}
    \begin{aligned}
        & \max_{\mathbb{X}, \mathbf{C}} \mcf(\mbx, \mbfc)=\max_{\mathbb{X}, \mathbf{C}} \frac{\sum_{ij} c_{ij} \cdot \vct(\mbfx_{ij})^{\top} \mbfk_{ij} \vct(\mbfx_{ij})}{\sum_{ij} c_{ij}} \\
        s.t. \hspace{5pt} & \mbfx_{ij} \mathbf{1}_{n_j} \leq \mathbf{1}_{n_i}, \mbfx_{ij}^{\top} \mathbf{1}_{n_i} \leq \mathbf{1}_{n_j}, \quad
        % & \mbfx_{ik} \mbfx_{kj} \leq \mbfx_{ij}, \forall c_{ik} = c_{kj} = c_{ij} = 1 \\
         c_{ik} c_{kj} \leq c_{ij}, \forall i,j,k, \hspace{2pt} \texttt{SCC}(\mbfc) = N_c,\\
         & \mbfx_{ij} \in \{0, 1\}^{n_i \times n_j}, c_{ij} \in \{0, 1\}.
    \end{aligned}
    \label{eq:mgmc_recap}
\end{equation}
Let the $\mbfc = h(\mbx)$ solve the optimal cluster division for $\mbx$, and $f(\mbx) = \mcf(\mbx, h(\mbx))$ be an objective function with single variable $\mbx$. Without loss of generality, we have
\begin{equation}
    \max_{\mbx} f(\mbx) = \max_{\mbx} \mcf(\mbx, h(\mbx)) = \max_{\mbx} \max_{\mbfc | \mbx} \mcf(\mbx, \mbfc) = \max_{\mbx, \mbc} \mcf(\mbx, \mbfc)
\end{equation}
Therefore, $f(\mbx)$ is a variant of the objective function.

The MM algorithm works by finding a surrogate function that minorizes the objective function $f(\mbx)$. Optimizing the surrogate function will either improve the value of the objective function or leave it unchanged. To optimize the surrogate function, two steps are conducted iteratively.
\begin{itemize}[leftmargin=*,itemsep=2pt,topsep=0pt,parsep=0pt]
\item \textbf{Construction.} The surrogate function $g(\mbx | \mbx^{(t)}) = \mcf(\mbx, h(\mbx^{(t)}))$ is constructed by inferring the best cluster based on current matching results $\mbx^{(t)}$:
\begin{equation}
    \begin{aligned}
        h(\mbx^{(t)}) &= \argmax_{\mathbf{C}} \sum_{ij} c_{ij} \cdot \vct(\mbfx_{ij}^{(t)})^{\top} \mbfk_{ij} \vct(\mbfx_{ij}^{(t)}) \\
        s.t. \hspace{5pt} & c_{ik} c_{kj} \leq c_{ij}, \forall i,j,k, \hspace{2pt} \texttt{SCC}(\mbfc) = N_c.
    \end{aligned}
    \label{eq:cluster_step}
\end{equation}
\item \textbf{Maximization.} Maximize $g(\mbx | \mbx^{(t)})$ instead of $f(\mbx)$, which is solved by off-the-shelf graph matching solver.
\begin{equation}
    \begin{aligned}
        \mathbb{X}^{(t+1)} 
        =& \argmax_{\mbx} g(\mbx | \mbx^{(t)}) = \argmax_{\mbx} f(\mbx, h(\mbx^{(t)})) \\
        =& \argmax_{\mathbb{X}} \sum_{ij} c_{ij}^{(t)} \cdot \vct(\mbfx_{ij})^{\top} \mbfk_{ij} \vct(\mbfx_{ij}) \\
        s.t. \hspace{5pt} &\mbfx_{ij} \mathbf{1}_{n_j} \leq \mathbf{1}_{n_i}, \mbfx_{ij}^{\top} \mathbf{1}_{n_i} \leq \mathbf{1}_{n_j}. \\
        % & \mbfx_{ik} \mbfx_{kj} \leq \mbfx_{ij}, \forall c^{(t)}_{ik} = c^{(t)}_{kj} = c^{(t)}_{ij} = 1.
    \end{aligned}
    \label{eq:matching_step}
\end{equation}
\end{itemize}
The \textbf{Construction} step ensures two properties that $g(\mbx | \mbx^{(t)})$ holds: 1) Since $f(\mbx) = \mcf(\mbx, h(\mbx))$ adopts the optimal cluster division, it holds for all $\mbx$ that
\begin{equation}\label{eq:g_p_1}
    g(\mbx | \mbx^{(t)}) = \mcf(\mbx, h(\mbx^{(t)}))  \leq \mcf(\mbx, h(\mbx)) =  f(\mbx)
\end{equation}
2) It also holds that
\begin{equation}\label{eq:g_p_2}
    g(\mbx^{(t)} | \mbx^{(t)}) = \mcf(\mbx^{(t)}, h(\mbx^{(t)})) =  f(\mbx^{(t)}).
\end{equation}
With Eq.~\ref{eq:g_p_1} and Eq.~\ref{eq:g_p_2}, for each iteration, the objective function will never decrease as shown in Eq.~\ref{eq:f_monotonic},
\begin{equation}
    f(\mbx^{(t+1)}) \ge g(\mbx^{(t+1)} | \mbx^{(t)}) \ge g(\mbx^{(t)} | \mbx^{(t)}) = f(\mbx^{(t)}).
\end{equation}
Thus we finish the proof of the convergence of our proposed MM Framework.

\subsection{Quick Convergence of Hard Clustering}
\label{sec:app:hard_cluster}
In this section, we show the proof of the proposition proposed in Sec.~\ref{sec:relaxed_indicator} that the hard clustering framework will converge quickly.

\begin{proposition}[Proposition~\ref{prop:cluster}]\label{prop:cluster_recap}
If the size of each cluster is fixed, the hard cluster indicator converges to the local optimum in one step:
\begin{equation}
\begin{aligned}
\mbfc^{(t)} = \mbfc^{(t+1)}, \text{ if }  \{N^{(t)}_{g_1}, N^{(t)}_{g_2}, \ldots, N^{(t)}_{g_{N_c}}\} = \{N^{(t+1)}_{g_1}, N^{(t+1)}_{g_2}, \ldots, N^{(t+1)}_{g_{N_c}}\}.
\end{aligned}
\end{equation}
\end{proposition}

\begin{proof}
    According to the maximization step in the MM framework, the pairwise matching is updated individually. The improvement on $\mbfx_{ij}$ does not influence the optimization of other pairs. Therefore, it holds that,
    \begin{equation}
        \label{eq:prop}
        \begin{aligned}
            J_{ij}^{(t + 1)} \ge J_{ij}^{(t)}, \ \ \forall c^{t}_{ij} = 1 \\
            J_{ij}^{(t + 1)} = J_{ij}^{(t)}, \ \ \forall c^{t}_{ij} = 0
        \end{aligned}
    \end{equation}
    where $J_{ij}=\vct(\mbfx_{ij})^{\top} \mbfk_{ij}\vct(\mbfx_{ij})$ denotes the pairwise affinity score.
    
    We prove the proposition by contradiction. Assume that $\mbc^{(t)} \neq \mbc^{(t+1)}$. According to the optimization framework, we have that
    \begin{equation}
        f(\mbx^{(t+1)}) \ge g(\mbx^{(t+1)} | \mbx^{(t)}) \ge g(\mbx^{(t)} | \mbx^{(t)}),
    \end{equation}
    which means,
    \begin{equation}
        \mcf(\mbx^{(t+1)}, \mbfc^{(t+1)}) \ge \mcf(\mbx^{(t+1)}, \mbfc^{(t)}) \ge \mcf(\mbx^{(t)}, \mbfc^{(t)}).
    \end{equation}
    Therefore,
    \begin{equation}
        \begin{aligned}
              & \mcf(\mbx^{(t+1)}, \mbfc^{(t+1)}) - \mcf(\mbx^{(t+1)}, \mbfc^{(t)}) \\
            = & \frac{1}{\sum c^{(t+1)}_{ij}} \sum_{ij} c^{(t+1)}_{ij} \cdot J_{ij}^{(t+1)} - \frac{1}{\sum c^{(t)}_{ij}} \sum_{ij} c^{(t)}_{ij} \cdot J_{ij}^{(t+1)} \\
            \ge & 0. \\ 
        \end{aligned}
    \end{equation}
    Since
    \begin{equation}
        \{N^{(t)}_{g_1}, N^{(t)}_{g_2}, \ldots, N^{(t)}_{g_{N_c}}\} = \{N^{(t+1)}_{g_1}, N^{(t+1)}_{g_2}, \ldots, N^{(t+1)}_{g_{N_c}}\},
    \end{equation}
    it holds that,
    \begin{equation}
        \sum c^{(t+1)}_{ij} =\sum c^{(t)}_{ij}.
    \end{equation} 
    According to Eq.~\ref{eq:prop}, we can further have
    \begin{equation}
        \begin{aligned}
            & \mcf(\mbx^{(t+1)}, \mbfc^{(t+1)}) - \mcf(\mbx^{(t+1)}, \mbfc^{(t)}) \\
            = & \frac{1}{\sum c^{(t)}_{ij}} \left( \sum_{ij} c^{(t+1)}_{ij} \cdot J_{ij}^{(t+1)} - \sum_{ij} c^{(t)}_{ij} \cdot J_{ij}^{(t+1)} \right) \\ 
            = & \frac{1}{\sum c^{(t)}_{ij}} \left( \sum_{c^{(t+1)}_{ij} = 1, c^{(t)}_{ij} = 0}  J_{ij}^{(t+1)} - \sum_{c^{(t+1)}_{ij} = 0, c^{(t)}_{ij} = 1} J_{ij}^{(t+1)}\right) \\ 
            = & \frac{1}{\sum c^{(t)}_{ij}} \left( \sum_{c^{(t+1)}_{ij} = 1, c^{(t)}_{ij} = 0}  J_{ij}^{(t)} - \sum_{c^{(t+1)}_{ij} = 0, c^{(t)}_{ij} = 1} J_{ij}^{(t+1)}\right) \\ 
            \ge & 0.
        \end{aligned}
    \end{equation}
    Moreover, according to $J_{ij}^{(t + 1)} \ge J_{ij}^{(t)}, \ \forall c^{t}_{ij} = 1$, we have
    \begin{equation}
        \begin{aligned}
         \sum_{c^{(t+1)}_{ij} = 1, c^{(t)}_{ij} = 0}  J_{ij}^{(t)} - \sum_{c^{(t+1)}_{ij} = 0, c^{(t)}_{ij} = 1} J_{ij}^{(t)} 
        \ge \sum_{c^{(t+1)}_{ij} = 1, c^{(t)}_{ij} = 0}  J_{ij}^{(t)} - \sum_{c^{(t+1)}_{ij} = 0, c^{(t)}_{ij} = 1} J_{ij}^{(t+1)},            
        \end{aligned}
    \end{equation}
    which means
    \begin{equation}
        \frac{1}{\sum c^{(t+1)}_{ij}} \sum_{ij} c^{(t+1)}_{ij} \cdot J_{ij}^{(t)} - \frac{1}{\sum c^{(t)}_{ij}} \sum_{ij} c^{(t)}_{ij} \cdot J_{ij}^{(t)} \ge 0.
    \end{equation}
    That is to say
    \begin{equation}
        \mcf(\mbx^{(t)}, \mbfc^{(t+1)}) \ge \mcf(\mbx^{(t)}, \mbfc^{(t)}),
    \end{equation}
    which means $\mbfc^{(t)}$ is not the optimal cluster division for $\mbfx^{(t)}$. Contradiction.
    Therefore, we have $\mbfc^{(t)} =  \mbfc^{(t+1)}$.
\end{proof}

\section{Detailed Algorithm of M3C}
Please refer to Algorithm~\ref{alg:main} and Algorithm~\ref{alg:um3c} for a detailed presentation of our proposed learning-free solver \name\ and unsupervised learning method \namenn.
\newcommand\mycommfont[1]{\footnotesize\ttfamily\textcolor{blue}{#1}}
\SetCommentSty{mycommfont}
\begin{algorithm}[tb!]
    \caption{Learning-free solver {\name}}
    \label{alg:main}
    \SetAlgoLined
    \DontPrintSemicolon
    \SetNoFillComment
	\KwIn{Iterations number $T$, Cluster number $N_c$, Graph set $\{ \mcg_1, \dots \mcg_N \}$.}
    Construct affinity matrix $\mbfk$ from node coordinates of graph set $\{ \mcg_1, \dots \mcg_N \}$ through standard process~{\citep{YanPAMI16, JiangPAMI21}}. \\
	Obtain initialization matching $\mbx^{(0)}$ via a two-graph solver e.g. RRWM~\citep{Cho2010ReweightedRW};\\
	\For{t = 1 : T}
	{
	    \tcc{Construction-Step}
	    Construct the surrogate function $g(\mbx | \mbx^{(t-1)})$ via solving $\hmbfc^{t} = \hat{h}(\mbx^{(t-1)})$ in Eq.~\ref{eq:relax_cluster_step} with three strategy candidates: global-rank, local-rank, or fuse-rank; \\
	   % Construct the surrogate function $g$ via solving $\hat{h}$ or $h$ in Eq.~\ref{eq:relax_cluster_step} (Eq.~\ref{eq:cluster_step}) with global-rank, local-rank, or fuse-rank (multicut or correlation clustering); \\
	    \tcc{Maximization-Step}
	    Set $\mbfa^{(t)}$ as relaxed indicator $\hmbfc^{(t)}$, maximizing $g(\mbx | \mbx^{(t-1)})$ via Eq.~\ref{eq:relax_matching_step} to find optimal matching composition $\mbx$ on $\mbfa^{(t)}$;\\
	}
	Obtain the affinity score for each graph pair: $J_{ij}=\vct(\mbfx_{ij})^{\top} \mbfk_{ij} \vct(\mbfx_{ij})$;\\
	Sparsification on the affinity score: $\set{{J}^k_{ij}} = \text{KNN}(\set{J_{ij}}, k)$;\\
	Apply clustering algorithm, e.g. Spectral Clustering~\citep{ng2002spectral}, Multi-Cut~\citep{swoboda2017message}, on $\set{{J}^k_{ij}}$ to get $\mbfc$.\\
    % By $\mbx^{(t)}$ and $\mbfk$, obtain hard clustering $\mbfc$ by multicut or clustering~\citep{ng2002spectral,kappes2016multicuts,swoboda2017message}; \\
	\KwOut{Matching $\mbx^{(t)}$, cluster $\mbfc$.}
\end{algorithm}
\setlength{\textfloatsep}{5pt}

\SetCommentSty{mycommfont}
\begin{algorithm}[tb!]
    \caption{Unsupervised learning {\namenn}}
    \label{alg:um3c}
    \SetAlgoLined
    \DontPrintSemicolon
    \SetNoFillComment
	\KwIn{Images $\set{\mci_1, \ldots, \mci_N}$, node coordinates $\set{\mcv_1, \ldots, \mcv_N}$}
    Obtain node and edge features $\mbff^n, \mbff^e$ via VGG16 and SplineCNN;\\
    Construct $\mbfk^{learn}$ from $\mbff^n, \mbff^e$ by $\mbfk^{learn}_{u} = (\mbff^{n})^{\top} \mathbf{\Lambda} \mbff^{n}, \mbfk^{learn}_{q} = (\mbff^{e})^{\top} \mathbf{\Lambda} \mbff^{e}$;\\
    Obtain hand-crafted $\mbfk^{raw}$ following \citet{ZhouPAMI16,YanPAMI16,JiangPAMI21,WangAAAI20};\\
    $\mbfk = \mbfk^{learn} + \alpha \mbfk^{raw}$; \\
    \uIf{training}
    {
    $\mbfx^{psd}, \hmbfc = \text{M3C}(\mbfk)$;\\
    $\mbfk^{psd} = \vct(\mbfx^{psd})\cdot \vct(\mbfx^{psd})$;\\
    $\mcl_{all} = \sum_{ij} \hc_{ij} \cdot \mcl(\mbfk_{ij}^{learn}, \mbfk_{ij}^{psd})$\\
    }
    \Else{
    $\mbfx, \mbfc = \text{M3C}(\mbfk)$;\\
    }
    \KwOut{Matching $\mbx$, cluster $\mbfc$.}
\end{algorithm}
\setlength{\textfloatsep}{5pt}

\section{Time Complexity Analysis between Learning-Free Solvers}
\label{sec:comp}
\begin{table}[tb!]
	\centering
	\caption{Time complexity comparison of peer methods, where $N$ and $n$ is the total number of graphs and number of nodes for each graph (we assume equal size here for notation simplicity), $T_i$ denotes the iterations needed for the algorithm to converge, and $\tau_{pair}$ denotes the time cost of calling a two-graph matching solver.}
    \label{tab:complexity}
    \resizebox{\textwidth}{!}{
        \begin{tabular}{llll}
        \toprule
        method & time complexity for MGMC                          \\ \midrule
        MGM-Floyd~\citep{JiangPAMI21}  & $\mco(N^4 n + N^3 n^3 + N^2 \tau_{pair} + tNN_cd)$     \\
        CAO-C~\citep{YanPAMI16}    & $\mco(N^4 n + N^3 n^3 + N^2 \tau_{pair} + tNN_cd)$         \\ 
        DPMC~\citep{WangAAAI20}    & $\mco(T_1 * (N^2 n^4 + N^2 \log N + tNN_cd) + N^2 \tau_{pair})$    \\
        GA-MGMC~\citep{wang2020graduated}  & $\mco(T_2 * (T' N^2 n^2 d + tNN_cd))$                    \\
        M3C (ours)    & $\mco(T_3 * (N^3 n^3 + N^2 \log N^2 + tNN_cd) + N^2 \tau_{pair})$              \\
        \bottomrule
        \end{tabular}
    }
\end{table}
We analyze the time complexity of our method \name\ and compare it with other learning-free solvers in Table~\ref{tab:complexity}. Notably, our algorithm outperforms vanilla MGM-Floyd~\citep{JiangPAMI21} in terms of speedup due to a lower time complexity bound. Additionally, our approach benefits from a significantly reduced constant factor.

Let $N$ and $n$ denote the number of graphs and nodes in one graph (ignoring the different graph sizes for the brevity of notation), respectively, and $\tau_{pair}$ denote the time cost of a two-graph matching solver, such as RRWM~\citep{Cho2010ReweightedRW}. It costs $\mco(N^2 \tau_{pair})$ to calculate the score matrix and $\mco(N^2\log N^2)$ to construct the supergraph. In the worst case, we would add $\frac{N(N-1)}{2}$ edges to the supergraph, resulting in $\mco(N^3 n^3)$ time cost for performing MGM-Floyd. Additionally, spectral clustering is applied, with a time cost of $\mco(tNN_cd)$, where $t$ is the k-means iteration, and $d$ is the dimension for embedded features. $T$ denotes the number of iterations that \name\ takes to converge. Therefore, the total time complexity is given by $\mco(T * (N^3 n^3 + N^2 \log N^2 + tNN_cd) + N^2 \tau_{pair})$.

It's important to note that \name\ runs faster than peer methods in Table~\ref{tab:main_test}. This is primarily due to two reasons. Firstly, \name\ exhibits a significantly lower constant factor in the maximization step, and the expected number of edges added is much smaller than in the worst case scenario. In practice, the actual run time of the maximization step can be approximated as $\theta N^3 n^3$, where $\theta < 1$. Empirical studies show that $\theta\approx 0.16$ for $N_c=5$ and $N_g=20$, and this value may further decrease with larger $N_c$ and $N_g$. Secondly, owing to the relaxed cluster indicator, \name\ converges much faster than GA-GAGM, resulting in a significantly reduced number of iterations: $T_3 \ll T_2$. As a result, the actual running time of \name\ is significantly less than that of DPMC and GA-MGMC.

\section{Implementation Details}
\label{sec:app:imp_detail}
In this section, we introduce more implementation details of \name\ and \namenn, including the detailed structure of the neural network we applied, the construction of affinity matrix $\mbfk$, and some hyper-parameters setting.

\subsection{Network Structure for Feature Extraction}
\begin{table*}[!b]
    \centering
    \caption{Network structure of vgg16\_bn as applied in \namenn. \yanq{The pre-trained weight is downloaded by PyTorch.} The bold line denotes the \texttt{relu4\_2}, \texttt{relu5\_1}, and \texttt{final} layers whose outputs are applied as node features, edge features, and global features.}
   \resizebox{0.8\linewidth}{!}{
    \begin{tabular}{|c|c|c|c|c|c|c|c|}
    \hline
                       & \textbf{Layer} & \textbf{Channels} & \textbf{Kernel} & \textbf{}           & \textbf{Layer}     & \textbf{Channels}   & \textbf{Kernel} \\ \hline
    \multirow{7}{*}{1} & Conv2d         & (3, 64)           & (3, 3)          & \multirow{2}{*}{3}  & ReLU               & (256, 256)          & -               \\
                       & BatchNorm2d    & (64, 64)          & -               &                     & MaxPool2d          & -                   & 2               \\ \cline{5-8} 
                       & ReLU           & (64, 64)          & -               & \multirow{10}{*}{4} & Conv2d             & (256, 512)          & (3, 3)          \\ \cline{2-4}
                       & Conv2d         & (64, 64)          & (3, 3)          &                     & BatchNorm2d        & (512, 512)          & -               \\
                       & BatchNorm2d    & (64, 64)          & -               &                     & ReLU               & (512, 512)          & -               \\ \cline{6-8} 
                       & ReLU           & (64, 64)          & -               &                     & \textbf{Conv2d}    & \textbf{(512, 512)} & \textbf{(3, 3)} \\ \cline{2-4}
                       & MaxPool2d      & -                 & 2               &                     & BatchNorm2d        & (512, 512)          & -               \\ \cline{1-4}
    \multirow{7}{*}{2} & Conv2d         & (64, 128)         & (3, 3)          &                     & ReLU               & (512, 512)          & -               \\ \cline{6-8} 
                       & BatchNorm2d    & (128, 128)        & -               &                     & Conv2d             & (512, 512)          & (3, 3)          \\
                       & ReLU           & (128, 128)        & -               &                     & BatchNorm2d        & (512, 512)          & -               \\ \cline{2-4}
                       & Conv2d         & (128, 128)        & (3, 3)          &                     & ReLU               & (512, 512)          & -               \\ \cline{6-8} 
                       & BatchNorm2d    & (128, 128)        & -               &                     & MaxPool2d          & -                   & 2               \\ \cline{5-8} 
                       & ReLU           & (128, 128)        & -               & \multirow{10}{*}{5} & \textbf{Conv2d}    & \textbf{(512, 512)} & \textbf{(3, 3)} \\ \cline{2-4}
                       & MaxPool2d      & -                 & 2               &                     & BatchNorm2d        & (512, 512)          & -               \\ \cline{1-4}
    \multirow{8}{*}{3} & Conv2d         & (128, 256)        & (3, 3)          &                     & ReLU               & (512, 512)          & -               \\ \cline{6-8} 
                       & BatchNorm2d    & (256, 256)        & -               &                     & Conv2d             & (512, 512)          & (3, 3)          \\
                       & ReLU           & (256, 256)        & -               &                     & BatchNorm2d        & (512, 512)          & -               \\ \cline{2-4}
                       & Conv2d         & (256, 256)        & (3, 3)          &                     & ReLU               & (512, 512)          & -               \\ \cline{6-8} 
                       & BatchNorm2d    & (256, 256)        & -               &                     & Conv2d             & (512, 512)          & (3, 3)          \\
                       & ReLU           & (256, 256)        & -               &                     & BatchNorm2d        & (512, 512)          & -               \\ \cline{2-4}
                       & Conv2d         & (256, 256)        & (3, 3)          &                     & ReLU               & (512, 512)          & -               \\ \cline{6-8} 
                       & BatchNorm2d    & (256, 256)        & -               &                     & \textbf{MaxPool2d} & \textbf{-}          & \textbf{2}      \\ \hline
    \end{tabular}
   }
    \label{tab:vgg16_bn}
\end{table*}

We utilize the feature extractor described in \citet{rolinek2020deep} with a few modifications. The process is outlined below:
\begin{itemize}[leftmargin=*,itemsep=2pt,topsep=0pt,parsep=0pt]
    \item Compute the outputs of $\texttt{relu4\_2}$, $\texttt{relu5\_1}$ of the VGG16~\citep{simonyanICLR14vgg} network pre-trained on ImageNet~\citep{krizhevsky2012imagenet}, to obtain feature $\mbff_1$ and $\mbff_2$, respectively. These features are then concatenated to create the final CNN feature $\mbff$:
    \begin{equation}
        \mbff = \texttt{CONCAT}(\mbff_1,\mbff_2)
    \end{equation} 
    The detailed network structure and parameters are shown in Table.~\ref{tab:vgg16_bn}.
    \item Feed the obtained feature $\mbff$ and the graph adjacency $\mca$ into the geometric feature refinement component. The graph adjacency $\mca$ is generated using Delaunay triangulation~\citep{delaunay1934sphere} based on keypoint locations. We apply SplineConv~\citep{FeyCVPR18} to encode higher-order information and the geometric structure of the entire graph into node-wise features $\mbff^{n}$:
    \begin{equation}
        \mbff^{n} = \texttt{SplineConv}(\mbff, \mca)
    \end{equation}
    The Spline Conv operation is calculated as follows:
    \begin{equation} 
        \mbff^{n}_i = \frac{1}{| \mathcal{N}(i)|} \sum_{j \in \mathcal{N}(i)} \mbff^{n}_j \cdot h_{\Theta}(\mbfe_{i,j})
    \end{equation} 
    where $\mbff^{n}_i$ represents the node feature of $v_i$, $\mathcal{N}(i)$ denotes the neighborhood of $v_i$, $\mbfe_{i,j}$ stands for the edge feature of the edge between $v_i$ and $v_j$, and $h_{\Theta}$ denotes a kernel function defined over the weighted B-Spline tensor product basis.
\end{itemize}

\subsection{Construction of Affinity Matrix $\mbfk$}
In our paper, we discuss two types of affinities: the traditional $\mbfk^{raw}$ and the learned $\mbfk^{learn}$. The former is employed in all learning-free experiments, while both are utilized in the \namenn\ process.

The construction of $\mbfk^{raw}$ adheres to the standard protocol used in previous works ~\citep{ZhouPAMI16,yan2015consistency,YanPAMI16,JiangPAMI21,WangAAAI20}. $\mbfk^{raw}$ has no node affinity and relies on edge affinity, which is computed based on two factors: length similarity and angle similarity. For each edge $e=((x_1, y_1), (x_2,y_2))$, its length feature $d_e$ is calculated as $d_e=\sqrt{(x_1-x_2)^2 + (y_1-y_2)^2}$, and its angle feature $\theta_e$ is computed as $\theta_e = \tan^{-1}(\frac{y_1-y_2}{x_1-x_2 + \epsilon})$. The edge affinity is determined using the following formula:
\begin{equation}
k^{raw}{e_1, e_2} = \exp \left(-\frac{1}{\sigma^2} \left(\beta \abs{d_{e_1} - d_{e_2}} + (1-\beta) \abs{\theta_{e_1} - \theta_{e_2}}\right)\right)
\end{equation}

The learned affinity $\mbfk^{learn}$ is extracted using a deep learning model, following the standard pipeline ~\citep{rolinek2020deep,wang2020graduated}. Node features $\mbff^{n}_{i}$ are obtained through VGG16 and SplineConv, while edge features are constructed as:
\begin{equation}
\mbff^{e}_{ij} = \mbff^{n}_{i} - \mbff^{n}_{j}
\end{equation}

The learned affinity matrix $\mbfk$ is computed as:
\begin{equation}
\begin{aligned}
\mbfk_{u} = (\mbff^{n})^{\top} \mathbf{\Lambda} \mbff^{n}, \quad
\mbfk_{q} = (\mbff^{e})^{\top} \mathbf{\Lambda} \mbff^{e}
\end{aligned}
\end{equation}
Here, $\mbfk_{u}$ represents unary affinity, $\mbfk_{q}$ denotes quadratic affinity, and $ \mathbf{\Lambda}$ is set to $\textbf{I}$ for stable training.
\namenn\ constructs the affinity matrix through a combination:
\begin{equation}\label{eq:weight_K}
\mbfk = \mbfk^{learn} + \alpha \mbfk^{raw},
\end{equation}
where $\alpha$ is used to adjust the weight of $\mbfk^{raw}$. Further parameter details are provided in Sec.~\ref{sec:parameter}.

\subsection{Parameter Settings}\label{sec:parameter}
\begin{table*}[tb!]
    \centering
    \caption{Parameters of UM3C.}
    \resizebox{\linewidth}{!}{
        \begin{tabular}{r|cc|cc|l}
        \toprule
        \multirow{2}{*}{param}    & \multicolumn{2}{c|}{WillowObject}        & \multicolumn{2}{c|}{PascalVOC} & \multirow{2}{*}{description}       \\
                                  & $3\times 8$       & $5 \times 10$        & $3 \times 8$ & $5 \times 10$ &                              \\ \midrule
        lr                        & $10^{-3}$         & $10^{-3}$            & $10^{-3}$ / $10^{-5}$ & $10^{-3}$ / $10^{-5}$ & learning rate                           \\
        lr-steps                  & $\set{100,500}$   & $\set{100,500}$      & $\set{100,500}$ & $\set{100,500}$ & lr/=10 at these number of iterations \\ \midrule
        \multirow{2}{*}{$\alpha$} & train: $1$        & train: $1$           & train: $0.5$  & train: $0.5$ & \multirow{2}{*}{weight of $\mbfk^{raw}$}    \\
                                  & test: $0$         & test: $1$            & test: $0.5$  & test: $0.5$ &                                         \\
        $\beta$                   & 0.9               & 0.9                  & 0.9 & 0.9 & weight parameter in $\mbfk^{raw}$               \\
        $\sigma^2$                & 0.03              & 0.03                 & 0.03 & 0.03 & the scaling factor in $\mbfk^{raw}$         \\
        $T$                       & 2                 & 2                    & 2 & 2 & max iterations of M3C                   \\ \bottomrule
        \end{tabular}
   }
    \label{tab:parameter_um3c}  
\end{table*}
The detailed configuration of our model parameters is listed in Table~\ref{tab:parameter_um3c}, which are tuned based on their performance on the training data. The parameter $\beta$ and $\sigma$ for $\mbfk^{raw}$ follows the parameter used for traditional solvers. The max iteration $T$ is chosen based on the performance, convergence, and time cost of \name. For $\alpha$, we found that as $\mbfk^{learn}$ gets better when the training proceeds, a less desirable $\mbfk^{raw}$ would harm the performance of the solver under the simpler setting where $N_c=3, N_g=8$, but is still instructive under more complex setting where $N_c=5, N_g=10$. This also means that given more training categories, there is still room for improvement in unsupervised learning methods.

\section{Experiment Details}

\subsection{Datasets}
\label{sec:datasets_detail}

% \todo{Check}
We conducted experiments on two datasets: Willow Object Class~\citep{ChoICCV13} and Pascal VOC Keypoint~\citep{PascalVOC}.

The Willow Object Class dataset comprises 304 images gathered from Caltech-256~\citep{caltech256} and Pascal VOC 2007~\citep{everingham2007pascal}. These images span 5 categories: 208 faces, 50 ducks, 66 wine bottles, 40 cars, and 40 motorbikes. Each image is annotated with 10 keypoints, and we introduced random outliers for robustness tests.

The Pascal VOC Keypoint dataset features natural images from 20 classes in VOC 2011~\citep{PascalVOC}, with additional keypoint labels provided by \citet{Bourdev2009PoseletsBP}. To tailor the dataset to the graph matching and clustering problem, we selected 10 classes: aeroplane, bicycle, bird, cat, chair, cow, dog, horse, motorbike, and sheep. We filtered out images with incomplete keypoint counts, ensuring that all remaining images had 9-10 common keypoints for each class. We also added random outliers to ensure that all images consistently contained exactly 10 keypoints. This resulted in a training set of 944 images and an evaluation set of 220 images.

In both datasets, we constructed graphs using Delaunay triangulation. For learning-based models, images were cropped to object bounding boxes and resized to $256 \times 256$ pixels.

\subsection{Evaluation Metric}
\label{sec:metric_detail}
Denote a cluster with a set of graphs $\mcc = \{ \mcg_1 \dots \mcg_{n} \}$. The ground truth cluster division is denoted as $\mcc^{gt}$, and the predicted cluster is denoted as $\mcc$. Moreover, let $\mbfc^{gt} (c^{gt}_{ij})$ denotes the ground truth cluster indicator and $\mbfc (c_{ij})$ denotes the predict cluster indicator. Performance metrics include both matching accuracy and clustering quality:

\paragraph{Matching Accuracy (MA)}  We only consider the intra-cluster matching accuracy and thus by adapting the accuracy for a single cluster, we have 
\begin{equation}
    \text{MA} = \frac{1}{\sum c^{gt}_{ij}} \sum_{ij} c^{gt}_{ij} \cdot \texttt{ACC}(\mbfx_{ij}), 
\end{equation}
where $\texttt{ACC}(\mbfx_{ij})$ denotes accuracy for matching $\mbfx_{ij}$. Here $\mbfc$ refers to an indicator for strict cluster division.

\paragraph{Clustering Purity (CP)}~\citep{schutze2008introduction}: it is given by 
\begin{equation}
    \text{CP} = \frac{1}{N} \sum_{i=1}^{N_c} \max_{j\in \set{1, \ldots, N_c}} \abs{\mcc_i \cap \mcc^{gt}_j},
\end{equation}
where $\mcc'_i$ is the predicted cluster $i$ and $\mcc_j$ is the ground truth cluster $j$, and $N$ is the total number of graphs.

\paragraph{Rand Index (RI)}~\citep{rand1971objective}: 
RI calculates the correct graph pairs overall.
\begin{equation}
    \text{RI} = \frac{1}{N^2} \cdot (\sum_{c_{ij}=1, c^{gt}_{ij}=1} 1 +  \sum_{c_{ij}=0, c^{gt}_{ij}=0} 1 )
\end{equation}
where $\sum_{c_{ij}=1, c^{gt}_{ij}=1} 1$ represents the number of graphs predicted in the same cluster with same label, $\sum_{c_{ij}=0, c^{gt}_{ij}=0} 1$ the number of pairs that are in different clusters with different labels, and it is normalized by the total number of graph pairs $N^2$.

\paragraph{Clustering Accuracy (CA)}~\citep{WangAAAI20}, it is defined by:
\begin{equation}
\text{CA} = 1 - \frac{1}{N_c}  \left( \sum_{\mcc_a} \sum_{\mcc'_a \neq \mcc'_b}  \frac{\abs{\mcc'_a \cap \mcc_a} \abs{\mcc'_b \cap \mcc_a}}{\abs{\mcc_a} \abs{\mcc_a}} \right. \left. + \sum_{\mcc'_a} \sum_{\mcc_a \neq \mcc_b} \frac{\abs{\mcc'_a \cap \mcc_a} \abs{\mcc'_a \cap \mcc_b}}{\abs{\mcc_a} \abs{\mcc_b}} \right)
\end{equation}
where $\mcc_a, \mcc_b$ are the ground truth clusters and $\mcc'_a, \mcc'_b$ denotes prediction.

\subsection{Apply 2GM and MGM on MGMC.}
To apply 2GM and MGM solvers on MGMC, we need to first get the pairwise matching results $\mbx$ from these solvers. Then we will generate respective clustering results $\mbfc$ based on resulting matching $\mbx$ and given affinity matrix $\set{\mbfk_{ij}}$. The details of the clustering algorithm are introduced in Sec.~\ref{sec:cluster}. Both $\mbx$ and $\mbfc$ are used for the evaluation of these solvers.

\subsection{Details for Clustering Algorithm.}\label{sec:cluster}
% \begin{wraptable}{r}{1\textwidth}
% \vspace{-25pt}
% \begin{minipage}[t]{1\linewidth}
\SetCommentSty{mycommfont}
\begin{algorithm}[!tb]
    \caption{Clustering algorithm.}
    \label{alg:cluster}
    \SetAlgoLined
    \DontPrintSemicolon
    \SetNoFillComment
	\KwIn{Matching results $\mbx$, affinity matrix $\set{\mbfk_{ij}}$.}
	Obtain the affinity score for each graph pair: $J_{ij}=\vct(\mbfx_{ij})^{\top} \mbfk_{ij} \vct(\mbfx_{ij})$;\\
	Sparsification on the affinity score: $\set{{J}^k_{ij}} = \text{KNN}(\set{J_{ij}}, k)$;\\
	Apply clustering algorithm, \eg{} Spectral Clustering~\citep{ng2002spectral}, Multi-Cut~\citep{swoboda2017message}, on $\set{{J}^k_{ij}}$ to get $\mbfc$.\\
	\KwOut{Cluster $\mbfc$.}
\end{algorithm}
\setlength{\textfloatsep}{5pt}
% \vspace{-20pt}
% \end{minipage}
% \end{wraptable}
For all solvers (2GM, MGM, and MGMC), we adhere to the same clustering procedure outlined in Alg.~\ref{alg:cluster}. The first step involves computing the affinity score $J_{ij}$ for each pair of graphs.
To enhance the effectiveness of clustering, we employ a sparsification technique, consistent with the pre-processing approach in \citet{WangAAAI20}, aimed at obtaining a more efficient input matrix. Specifically, when dealing with a pair of two graphs, if one graph is not among the $k$-nearest neighbors of the other, we set their corresponding scores $J_{ij}$ (and $J_{ji}$) to zero. The parameter $k$ is consistently set to $10$ for all tests. The resulting sparsified affinity score is denoted as $J^k_{ij}$.

\section{Additional Experiments}

\subsection{Varying Cluster Number and Cluster Size}
\label{sec:varying_setting}

\begin{table*}[b!]
	\centering
	\caption{Evaluation of matching and clustering accuracy by varying the number of clusters, and number of graphs in each cluster on WillowObj. MA and CA are used for matching accuracy and clustering accuracy, respectively.}
    \label{tab:varying_setting}
    % %\vspace{-5pt}
	    \resizebox{\textwidth}{!}
    {
\begin{tabular}{l|c|cc|cc|cc|cc|cc|cc}
\toprule[2pt]
$N_c \times N_g$                &          & \multicolumn{2}{c|}{$3\times 20$, 2 outliers} & \multicolumn{2}{c|}{$4\times 20$, 2 outliers} & \multicolumn{2}{c|}{$5\times 20$, 2 outliers} & \multicolumn{2}{c|}{$5\times 15$, 2 outliers} & \multicolumn{2}{c|}{$5\times 10$, 2 outliers} & \multicolumn{2}{c}{$3\times [20,10,5]$, 2 outliers} \\
Metrics                         & Learning & MA $\uparrow$         & CA $\uparrow$         & MA $\uparrow$         & CA $\uparrow$         & MA $\uparrow$         & CA $\uparrow$         & MA $\uparrow$         & CA $\uparrow$         & MA $\uparrow$         & CA $\uparrow$         & \ \ \ \ MA $\uparrow$        & CA $\uparrow$        \\ \midrule
RRWM & free     & 0.658                 & 0.932                 & 0.642                 & 0.858                 & 0.665                 & 0.790                 & 0.648                 & 0.693                 & 0.664                 & 0.679                 & \ \ \ \ 0.633                & 0.681                \\
CAO-C          & free     & 0.849                 & 0.946                 & 0.812                 & 0.855                 & 0.820                 & 0.790                 & 0.801                 & 0.708                 & 0.804                 & 0.679                 & \ \ \ \ 0.787                & 0.757                \\
MGM-Floyd    & free     & 0.845                 & 0.945                 & 0.812                 & 0.878                 & 0.819                 & 0.807                 & 0.798                 & 0.727                 & 0.799                 & 0.707                 & \ \ \ \ 0.778                & 0.755                \\
DPMC          & free     & \textbf{0.867}        & 0.942                 & 0.827                 & 0.894                 & 0.775                 & 0.772                 & 0.739                 & 0.713                 & 0.756                 & 0.744                 & \ \ \ \ \textbf{0.795}       & 0.823                \\
\name-hard                   & free     & 0.758                 & \textbf{0.966}        & 0.782                 & 0.908                 & 0.726                 & 0.824                 & 0.710                 & 0.753                 & 0.722                 & 0.719                 & \ \ \ \ 0.727                & 0.744                \\
\textbf{\name~(ours)}           & free     & 0.857                 & 0.961                 & \textbf{0.851}        & \textbf{0.933}        & \textbf{0.835}        & \textbf{0.900}        & \textbf{0.812}        & \textbf{0.805}        & \textbf{0.809}        & \textbf{0.780}        & \ \ \ \ 0.792                & \textbf{0.881}       \\ \midrule[2pt]
GANN   & unsup.   & 0.532                 & 0.834                 & 0.589                 & 0.801                 & 0.528                 & 0.784                 & 0.551                 & 0.802                 & 0.552                 & 0.827                 & \ \ \ \ 0.475                & 0.802                \\
\textbf{\namenn~(ours)}         & unsup.   & \textbf{0.874}        & \textbf{0.992}        & \textbf{0.897}        & \textbf{0.981}        & \textbf{0.879}        & \textbf{0.972}        & \textbf{0.876}        & \textbf{0.975}        & \textbf{0.878}        & \textbf{0.975}        & \ \ \ \ \textbf{0.872}       & \textbf{0.984}       \\ \bottomrule[2pt]
\end{tabular}
    }   
   % %\vspace{-10pt}
\end{table*}

We assess the model's generalization ability concerning the number of graphs and clusters. For $N_c = 3$, the categories consist of car, motorbike, and wine bottle. For $N_c = 4$, additional categories include face. We also investigate unbalanced cluster sizes for $N_c = 3$, comprising 20 cars, 10 motorbikes, and 5 wine bottles. Both GANN and \namenn\ are trained with $N_c=5$ and $N_g=10$, excluding outliers. During testing, two outliers are randomly added to the graph in all settings.

Table~\ref{tab:varying_setting} demonstrates the robustness of our methods with varying cluster and graph numbers.	
Our learning-free solver, \name, exhibits competitive performance compared to DPMC, with a matching accuracy ranging from $1\%$ loss to $2\%$ gain, and clustering accuracy improvement ranging from $2\%$ to $9\%$.	
These achievements further reflect the superiority of our proposed ranking scheme.

Notably, our \namenn\ outperforms in all performance metrics.	
It consistently achieves a cluster accuracy above $0.97$ and a matching accuracy exceeding $0.87$ across all settings, representing a $1\sim 7\%$ improvement in matching accuracy and $2\sim 18\%$ enhancement in clustering accuracy. 
The fact that \namenn's training setting differs from testing settings validates the strong generalization ability of our method across varying numbers of graphs, clusters, and the presence of outliers. Furthermore, our method, although trained under simpler conditions, can be effectively deployed in more complex scenarios, yielding satisfactory performance.	

\subsection{Comparison of Different Ranking Schemes}
\label{sec:rank_schemes}

\begin{figure*}[tb!]
    \centering
    \includegraphics[width=\textwidth]{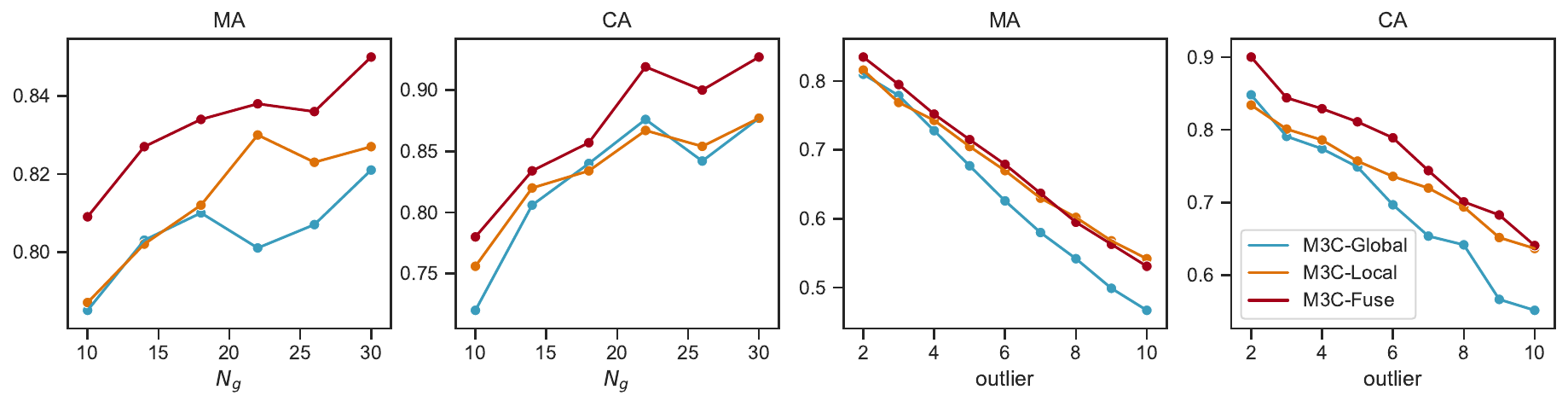}
    %\vspace{-15pt}
    \caption{Comparison by four metrics of the three proposed ranking schemes (local, global, and fuse in Section~\ref{sec:our}) on the Willow ObjectClass dataset, by varying the cluster size $N_g$ and number of outliers.}
    \label{fig:self_comp}
    % %\vspace{-5pt}
\end{figure*}
We compare three proposed ranking schemes across various test settings.	
\threshold, \KNN, and \Ranksum\ refer to global-rank, local-rank, and fuse-rank, respectively.
We vary the number of graphs $N_g$ within each cluster while keeping the cluster number fixed at $N_c=5$, and we also vary the number of outliers in each graph while maintaining $N_c=5$ and $N_g=20$.

Our findings from Fig.~\ref{fig:self_comp} indicate that \Ranksum\ outperforms all other methods, leading us to select \Ranksum\ as the solver for our unsupervised model, \namenn. These results also confirm that both global and local ranking schemes serve as effective approximations. Furthermore, this demonstrates the robustness and generalization ability of our ranking methods. Even in the presence of $10$ outliers, our method achieves a matching accuracy exceeding $0.5$ and a clustering accuracy surpassing $0.65$. Additionally, the performance of our methods improves in both matching and clustering accuracy as the number of graphs increases. This observation also explains why \name\ does not outperform other learning-free solvers in Table~\ref{tab:main_test} (in simpler settings) but demonstrates significant superiority in Table~\ref{tab:varying_setting} (in more complex settings).

\subsection{Comparison of Different Clustering Algorithms}
\label{sec:comparison_of_clustering_alg}
In previous experiments, for all solvers (under the settings of 2GM, MGM, and MGMC), we adopt the same procedure for clustering. The first step involves computing the affinity score $J_{ij}$ for each pair of graphs. To sparsify the affinity scores, we select only the 10 nearest neighbors for each graph and mask other pairwise affinities, following the approach in \citet{WangAAAI20} to obtain a more effective input matrix. Subsequently, we employ a clustering algorithm on the sparse affinity $\set{J^k_{ij}}$. 

We now conduct a comparison between the clustering algorithms Spectral Clustering~\citep{ng2002spectral} and Multi-Cut~\citep{swoboda2017message} applied to two well-established traditional algorithms: MGM-Floyd and \Ranksum. We hope this comparison justifies our choice of clustering algorithm.

Table~\ref{tab:clustering} presents the performance of all four combinations. As a result, there is no substantial alternation in clustering performance. As both~\citep{WangAAAI20, wang2020graduated} utilized spectral clustering, to ensure a fair comparison, we adhere to their protocol and employ spectral clustering in our primary experiments.

Furthermore, we posit that the key to achieving effective clustering lies in obtaining high-quality matching and forming reliable affinity scores for clustering. Multi-Cut, as well as Spectral Clustering, represents just one approach to produce robust clustering. The clustering visualization of different methods is shown in Fig.~\ref{fig:vis_cluster}.

\begin{table*}[tb!]
    \centering
    \caption{Comparison of Spectral Clustering and Multi-Cut on learning-free solvers under the setting of $N_C=3$, $N_g=8$, and $n_o=2$ outliers. Algorithms with `-MC' use multicut in clustering.}
  %  \resizebox{\linewidth}{!}{
        \begin{tabular}{cc|cccc}
        \toprule
        \multicolumn{2}{c|}{Method}                                                   & MA $\uparrow$ & CA $\uparrow$ & CP $\uparrow$ & RI $\uparrow$ \\ \midrule
        \multicolumn{1}{c|}{\multirow{2}{*}{With Spectral Clustering}} & MGM-Floyd    & 0.709         & 0.567         & 0.673         & 0.699         \\
        \multicolumn{1}{c|}{}                                          & \name       & 0.687         & 0.653         & 0.750         & 0.758         \\ \midrule
        \multicolumn{1}{c|}{\multirow{2}{*}{With Multi-Cut}}           & MGM-Floyd-MC & 0.709         & 0.603         & 0.716         & 0.724         \\
        \multicolumn{1}{c|}{}                                          & \name-MC    & 0.687         & 0.634         & 0.734         & 0.745         \\ \bottomrule
        \end{tabular}
   % }
    \label{tab:clustering}   
\end{table*}

\begin{figure*}[tb!]
    \begin{center}
        \includegraphics[width=0.85\textwidth]{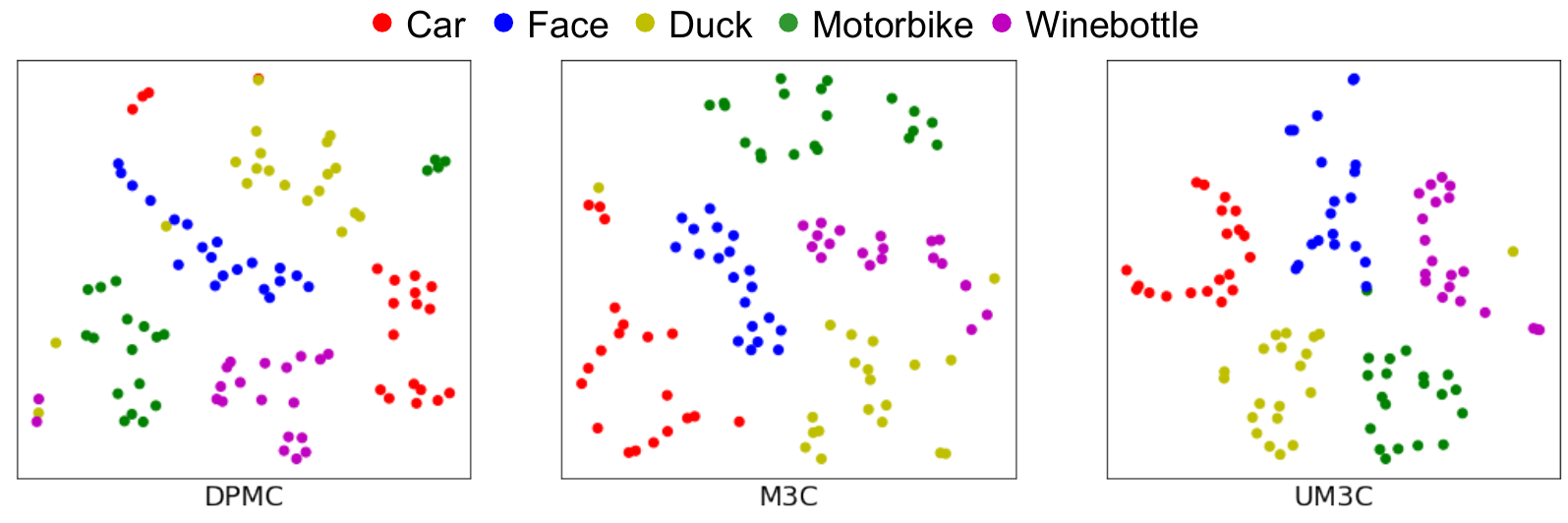}
    \end{center}
    %\vspace{-10pt}
    \caption{Cluster visualization by projecting into 2-D space. We show the spectral embedding of different methods: DPMC, M3C, and UM3C under $5\times 20, 2$ outliers setting. The embedding is obtained based on pairwise affinity score and the dimension of the embedding space is 16. We apply t-SNE to reduce the dimension to 2 to draw the visualization figures. %Points are colored to illustrate their ground truth category.
    }
    \label{fig:vis_cluster}
\end{figure*}

\subsection{Convergence Study of M3C}
\label{sec:convergence_study}
\begin{wrapfigure}{r}{0.45\linewidth}
\begin{center}
	\includegraphics[width=\linewidth]{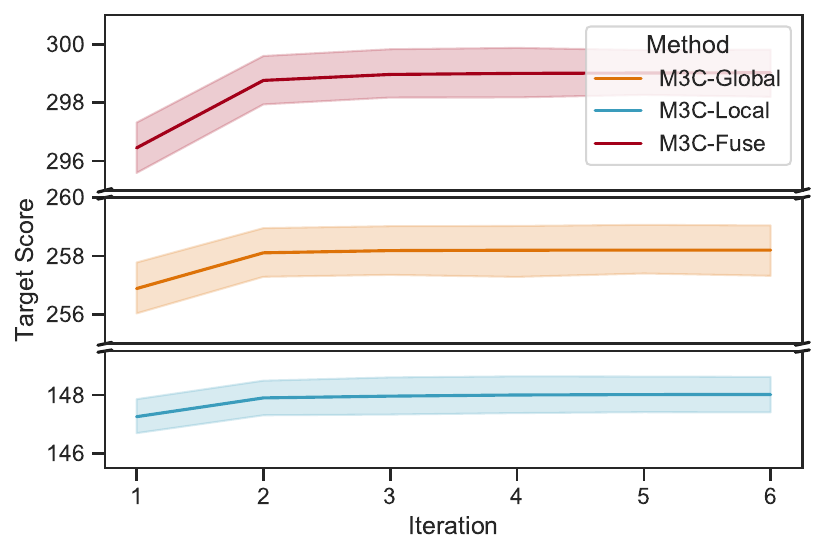}
\end{center}
\caption{Convergence curve of M3C. The rates are shown under three different schemes of ranking: \threshold, \KNN, and \Ranksum.}
\label{fig:convergence}
\end{wrapfigure}

We experiment to show the changes in the supergraph structure of \name-hard, \name, and DPMC per iteration in Table~\ref{tab:structure}.  
In the case of the two \name\ variants, the structure refers to the cluster indicator (also the corresponding supergraph), whereas for DPMC, it pertains to the designed tree structure. The number is the edges changed per iteration, which is calculated by $\sum |\mbfa^{(t+1)} - \mbfa^{(t)}|$ where $\mbfa$ is the adjacency matrix of the respective supergraph. 
It is evident that DPMC oscillates without convergence, while \name-hard converges rapidly to a local optimum. \name\ exhibits a more balanced convergence rate, leading to its well-balanced performance.

\begin{table*}[tb!]
\begin{center}
    \caption{Changes in supergraph structure (measured by the number of changed edges per iteration $\sum \abs{\mbfa^{(t+1)} - \mbfa^{t}}$) over iterations under the setting of $N_c=3$, $N_g=8$ with $2$ outliers as disturbance. For \name, the structure is the cluster indicator $\hmbfc$, and for DPMC, the structure is the maximum spanning tree.} 
    %\yanr{need to polish the writing which is very vague!!!}}
   % \resizebox{0.5\linewidth}{!}{
        \begin{tabular}{l|ccccccccc}
        \toprule
        Iteration \#   & 2     & 3    & 4    & 5    & 6    & 7    & 8    & 9    & 10   \\ \midrule
        \name-hard & 10.48 & 0.56 & 0    & 0    & 0    & 0    & 0    & 0    & 0    \\
        \name & 20.44 & 2.04 & 1.56 & 0.24 & 0.04 & 0.08 & 0    & 0    & 0    \\
        DPMC & 10.16 & 6.16 & 3.28 & 1.20 & 0.48 & 0.32 & 0.24 & 0.24 & 0.24 \\
        \bottomrule
        \end{tabular}
   % }
    \label{tab:structure}
\end{center}
\end{table*}

Additionally, we validate the convergence of the three \name\ variants using different ranking schemes: \name-Global, \name-Local, and \name-Fuse. The experiment is conducted on the Willow ObjectClass dataset, with the settings of $N_c=5$, $N_g=20$, and the presence of two outliers. We iterate each algorithm for $6$ cycles and report the mean and standard deviation of the curves based on $50$ repetitions. The hyperparameters $r$ for \threshold, \KNN, and \Ranksum\ are set to $0.05$, $0.04$, and $0.06$, respectively. The results are depicted in Fig.~\ref{fig:convergence}. They validate that our algorithm achieves rapid convergence within a few iterations. In the case of each algorithm, it attains a near-optimal target score by the second iteration.	This supports the earlier assessment of supergraph structure convergence: the second iteration witnesses a significant number of edge changes, which diminishes in the third iteration but still allows room for further enhancement.	It is important to note that variations in target scores are a consequence of selecting different values of $r$ for each algorithm.

\subsection{Hyperparameter Study of M3C}
\label{sec:hyperparameter}

\begin{figure*}[tb!]
    \centering
    \includegraphics[width=\textwidth]{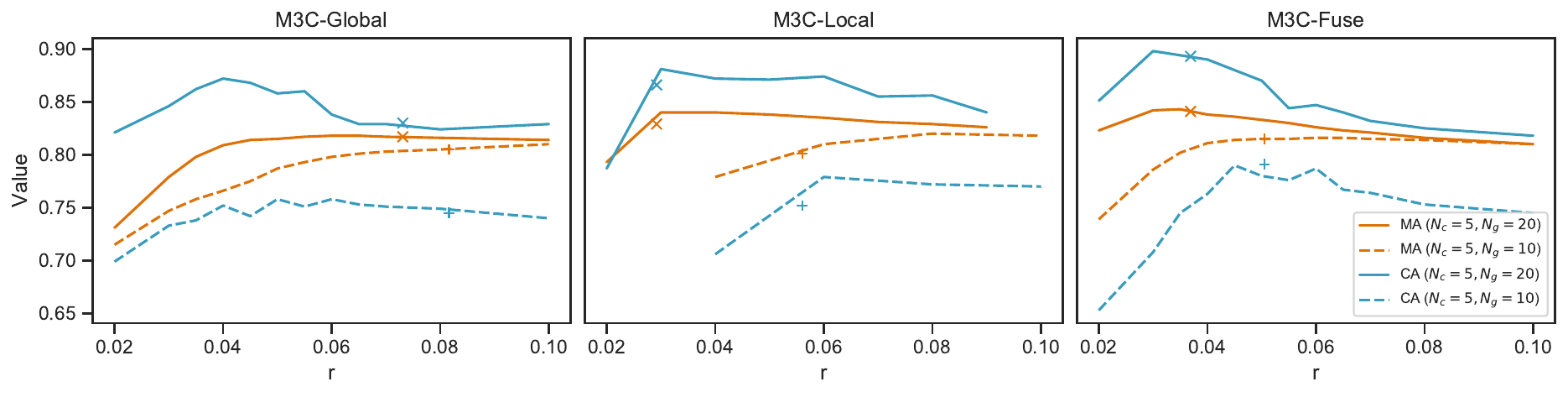}
    \caption{Sensitivity study of the hyperparameter $r$ for the devised three ranking schemes. Experiments are conducted under the setting $N_g=5$, $N_c=20$ and $N_g=5$, $N_c=10$, both with $2$ outliers, on Willow ObjectClass. The marker denotes the performance of our chosen $r$, which is to add edges until the supergraph is connected.}
    \label{fig:parameter_r}
\end{figure*}

The major hyperparameter for \name\ is $r$, which controls the number of graph pairs considered as belonging to the same cluster and determines the number of edges in the supergraph.
In this section, we first investigate the sensitivity of the hyperparameter $r$ for \name-Global, \name-Local, and \name-Fuse, and subsequently, we present our tuning algorithm.

Figure~\ref{fig:parameter_r} illustrates the matching and clustering performance varying the hyperparameter $r$, considering two settings: $N_c=5$, $N_g=20$, and $N_c=5$, $N_g=10$, each with $n_o=2$ outliers. It is evident that the matching performance of \threshold\ remains stable when $r>0.04$ for $N_c=5$, $N_g=20$, and $r > 0.06$ for $N_c=5$, $N_g=10$. However, its clustering performance deteriorates when $r\geq 0.6$ in both settings. This observation implies that the threshold should be within a reasonable range, as merely adding more edges does not necessarily improve performance. Conversely, having too few edges restricts the algorithm's optimization space. The findings from \KNN\ and \Ranksum\ further support this observation. As depicted in Fig.~\ref{fig:parameter_r}, they achieve optimal results at $r=0.15$ and $r=0.03$ for the $5\times 20$ setting, and $r=0.3$ and $r=0.045$ for the $5\times 10$ setting.	

Additionally, it is worth noting that the optimal $r$ varies for different inputs and settings, and determining the best $r$ for each input can be a time-consuming process. Consequently, we employ an alternative approach to address this challenge. Rather than fixing a specific value for $r$, we dynamically add edges based on their rank until the supergraph becomes connected. The symbols `$\times$' and `$+$' in Fig.~\ref{fig:parameter_r} represent the mean value of $r$ and the corresponding mean performance achieved by this scheme in two settings, respectively.	These empirical results demonstrate that this approach provides a reliable approximation of the optimal $r$, enabling the algorithm to attain near-optimal performance without extensive computation. This is the method employed in both our conventional solver, \name, and the unsupervised learning method, \namenn.

\subsection{Generalization Test of learned affinity $\mbfk^{learn}$}
\label{sec:generalization_of_k}

\begin{figure}[tb!]
    \centering
    \includegraphics[width=0.7\linewidth]{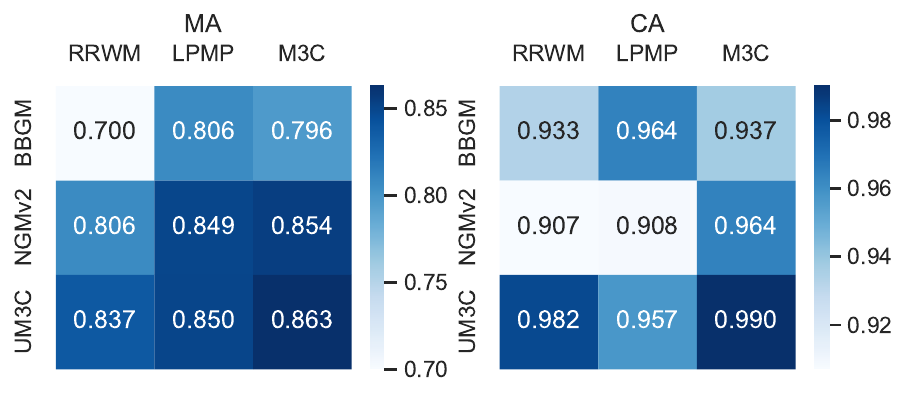}
    \caption{Generalization test of learned affinity under $N_c=3$ $N_g=8$, $n_o=2$ outliers on WillowObject. The x-axis is the solver used, and the y-axis is the learning model of the learned affinity. All models are learned under the same setting as the tests.}
    \label{fig:gen}
\end{figure}
We conducted experiments to evaluate the generalization capability of our learned affinity, demonstrating how our edge-wise affinity loss enhances the robustness of affinity across different solvers. The experiments were conducted on Willow ObjectClass under a $3\times 8$ setting with 2 outliers, employing the solvers RRWM~\citep{Cho2010ReweightedRW} (used in \name), LPMP~\citep{swoboda2017study} (utilized in BBGM), and \name, as well as the affinities learned by BBGM, NGMv2, and \namenn. In the case of \namenn, only $\mbfk^{learn}$ was utilized for testing.

% \begin{wrapfigure}{r}{0.45\linewidth}
% \begin{center}
%     \includegraphics[width=\linewidth]{fig/fig_generalization.pdf}
%     \caption{Generalization test of learned affinity under $N_c=3$ $N_g=8$, $2$ outliers on WillowObject. The x-axis is the solver used, and the y-axis is the learning model of the learned affinity. All models are learned under the same setting as the tests.}
%     \label{fig:gen}
% \end{center}
% \end{wrapfigure}
As illustrated in Fig.~\ref{fig:gen}, \namenn\ exhibited superior generalization capabilities in terms of both matching and clustering accuracy. BBGM's learning pipeline limited its applicability to the LPMP solver, offering less utility for other solvers. Furthermore, the affinity generated by \namenn\ significantly improved the performance of LPMP compared to that generated by BBGM. These results affirm that our edge-wise affinity learning enhances the robustness of the learned affinity, making it adaptable to various solvers.

\section{Visualizations}
\label{sec:more_visualization}
\begin{figure*}[h]
    \begin{center}
        \includegraphics[width=\textwidth]{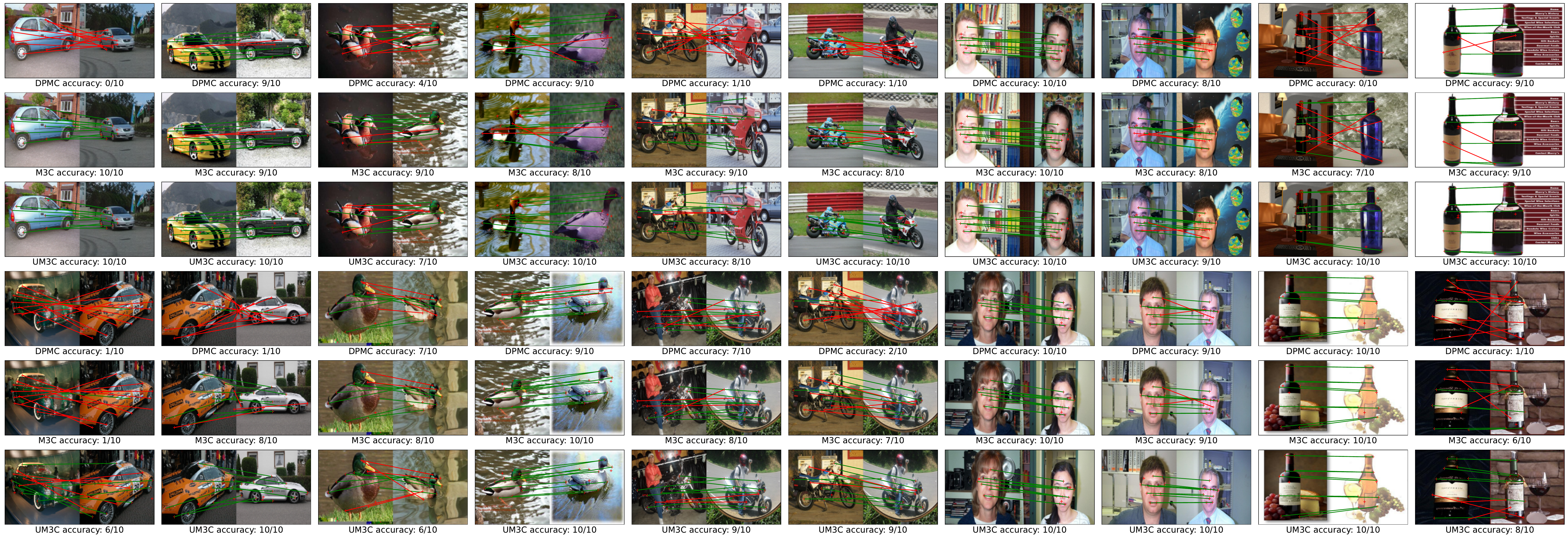}
    \end{center}
    \vspace{-10pt}
    \caption{Comparison of different methods: DPMC(top), M3C(Middle), and UM3C(bottom). It is run on the setting with $N_c=5$ and $N_g=20$ and 2 outliers. Accuracy is reported for each pairwise matching. All the pairs are randomly picked. Better viewing with color and zooming in.}
    \label{fig:vis_method_compare}
\end{figure*}

\begin{figure*}[h]
    \begin{center}
        \includegraphics[width=\textwidth]{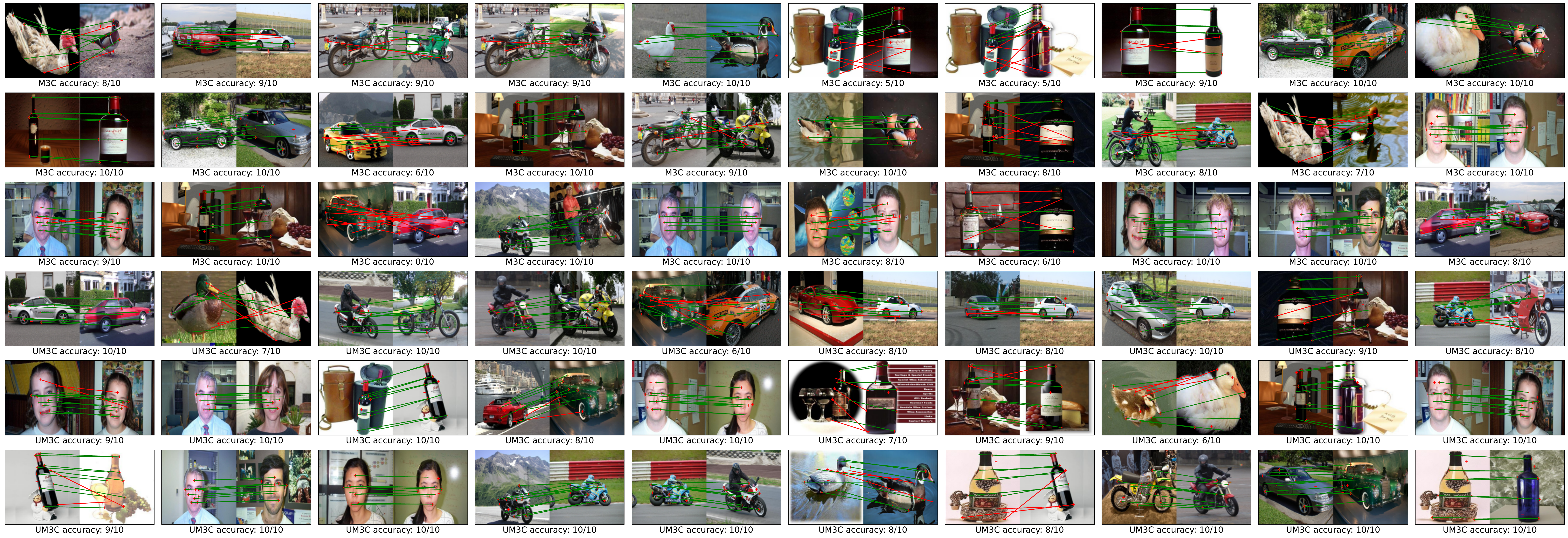}
    \end{center}
    \vspace{-10pt}
    \caption{Visualization of our methods: M3C(top) and UM3C(bottom). It is run on the setting with $N_c=5$ and $N_g=20$ graphs and 2 outliers. Accuracy is reported for each pairwise matching. All the pairs are randomly picked. Better viewing with color and zooming in.}
    \label{fig:vis_our_method}
\end{figure*}

\end{document}